\documentclass[runningheads]{llncs}
\usepackage{caption}
\usepackage{subcaption}
\usepackage{booktabs}
\usepackage[
  colorlinks=false,
  pdfborder={0 0 1},
  citebordercolor={0 1 0}
]{hyperref}\usepackage{graphicx}
\usepackage[ruled,vlined]{algorithm2e}
\usepackage{tikz}
\usepackage{xcolor}
\usepackage{listings}
\usepackage{xcolor}
\usepackage{hyperref}

\usepackage{pdflscape}
\usetikzlibrary{positioning}       
\usepackage{amsmath}
\usepackage{listings}
\usepackage{longtable}
\usepackage{amssymb}
\usetikzlibrary{shapes.geometric,automata,arrows.meta, positioning, fit, backgrounds, calc, decorations.pathreplacing}

\usepackage{appendix}
\title{Warm Starting State-Space Models with Automata Learning}
\author{William Fishell\inst{1} \and Sam Nicholas Kouteili\inst{2} \and 
 Mark Santolucito\inst{1}$^\dagger$}
\authorrunning{Fishell et al.}

\institute{
Columbia University ($\dagger$ Barnard College, Columbia University) \and
Yale University 
}

\begin{document}
\maketitle
\begin{abstract}
We prove that Moore machines can be exactly realized as state-space models (SSMs), establishing a formal correspondence between symbolic automata and these continuous machine learning architectures.
These Moore-SSMs preserve both the complete symbolic structure and input-output behavior of the original Moore machine, but operate in Euclidean space.
With this correspondence, we compare the training of SSMs with both passive and active automata learning.
In recovering automata from the SYNTCOMP benchmark, we show that SSMs require orders of magnitude more data than symbolic methods and fail to learn state structure. 
This suggests that symbolic structure provides a strong inductive bias for learning these systems.
We leverage this insight to combine the strengths of both automata learning and SSMs in order to learn complex systems efficiently. 
We learn an adaptive arbitration policy on a suite of arbiters from SYNTCOMP and show that initializing SSMs with symbolically-learned approximations learn both faster and better.
We see 2-5 times faster convergence compared to randomly initialized models and better overall model accuracies on test data.
Our work lifts automata learning out of purely discrete spaces, enabling principled exploitation of symbolic structure in continuous domains for efficiently learning in complex settings.
\end{abstract}

\section{Introduction}
Symbolic learning methods such as active and passive automata learning~\cite{angluin1987learning,tirnuaucua2012survey} are highly effective for systems that admit finitely representable behavioral models. These methods have been widely applied to infer models of network and communication protocols~\cite{aichernig2022active}, where they infer behavioral models by treating the target system as a black box. Automata learning recovers symbolic structure: explicit, discrete representations of state, transitions, and outputs that fully characterize system behavior over a finite set of configurations. However, these methods face key limitations: passive learning struggles to scale to complex problems, while active learning is bottlenecked by the cost of membership and equivalence queries~\cite{aichernig2024benchmarking}. The discrete nature of automata compounds these challenges, as there is no notion of proximity between behavioral models; solving a simpler problem provides no initialization for learning a related but more complex target. Furthermore, systems that track things about cumulative quantities, such as version control histories or API request counts, are fundamentally beyond the reach of classical automata learning. These methods recover finite-state representations, yet systems where behavior depends on the entire unbounded history require infinite memory.

We address these problems by formalizing the relationship between Moore machines and state-space models (SSMs)~\cite{hamilton1994state}. SSMs are continuous-state recurrent models that process sequences through learned linear dynamics. 
 SSMs such as Mamba~\cite{gu2024mamba} have received growing interest 
as they present a computationally efficient alternative to Transformer architectures that can compute linear attention.
Beyond efficiency, we find that SSMs can realize Moore machines exactly.
Representing Moore machines as SSMs yields three benefits: we inherit symbolic structure as an inductive bias, gain access to a continuous space where proximity can be exploited, enabling gradient-based refinement toward more complex settings, and can leverage recurrence to more easily model infinite-state systems. This correspondence provides an end-to-end pipeline where automata learning extracts structure that initializes continuous models, extending the reach of symbolic methods and enabling more efficient learning in deep learning settings.

We evaluate our approach of symbolic warm-starting on a suite of arbiters from SYNTCOMP. These systems have an underlying control policy described by a finite-state arbiter. We then augment this base policy such that the arbiter must track each channel's deviation from the average number of grants. Computing this deviation requires tracking grants across the entire history, demanding theoretically infinite memory. Such systems cannot be captured exactly by finite-state machines, making them natural candidates for neural approaches initialized with symbolic structure.

To motivate this hybrid approach, we first establish the value of symbolic structure in purely finite-state settings. We evaluate sample efficiency of symbolic and neural methods on a suite of regular languages from SYNTCOMP. Active learning (L*) and passive learning (RPNI) are compared against randomly initialized SSMs trained via gradient descent. This baseline experiment measures how many samples each method requires to recover a target automaton's behavior. Automata learning achieves significantly greater sample efficiency than neural approaches trained from scratch. This motivates our use of symbolic structure as initialization for learning systems that extend beyond finite-state representations\footnote{All code available at \href{https://github.com/wfishell/Automata_SSM_Learning/tree/main}{https://github.com/wfishell/Automata\_SSM\_Learning/tree/main}}.
In all, the key contributions of this work are as follows:
\begin{enumerate}
\item To the best of our knowledge, we are the first to initialize SSMs from automata recovered via classical learning algorithms, enabling more sample-efficient learning of complex settings.
\item We formalize a proof that Moore machines admit exact realizations as SSMs, preserving both structure and behavior.
\item We present an empirical study on SYNTCOMP benchmarks showing that symbolic methods achieve orders of magnitude greater sample efficiency than gradient-trained SSMs.
\end{enumerate}

\section{Motivating Example}
\label{sec:motiv}
Cloud resource allocation is an important task in managing customer access to cloud resources by providers like Amazon Web Services (AWS)~\cite{almarhabi2024distributed}.
Modeling and verifying allocation protocols has been a focus of formal methods research~\cite{boubaker2016formal,fan2016formal,garfatta2018formal}. 
Additionally, machine learning techniques have been increasingly used for modeling more complex resource allocation schemes~\cite{zi2024time,sharara2021recurrent}. Furthermore, warm-starting—initializing a model with parameters learned from a simpler setting before training on more complex instances—has also proven effective for cloud resource allocation.
Prior work uses warm-starting in neural models for cloud resource allocation and relies on simpler neural architectures and unsupervised methods~\cite{peng2021dl2}.
Our work bridges formal and neural approaches by using symbolic structure to warm-start neural models for efficient learning of adaptive arbitration policies in cloud resource allocation.

\begin{example}[GPU Allocation System]~\label{example:GPU_Allocation}
Company A is a cloud GPU provider that has four customers and an 8-GPU cluster. Company A has a policy for granting access to their network based on a round robin arbiter with the added safety constraint that no user can request more than 25\% of the network's capacity. Once a GPU is allocated, it is live for the entire hour. Afterward, the user is kicked off the network. 
\end{example}
\begin{table}[t]
\centering
\tiny
\label{tab:gpu-allocation-combined}

\begin{subtable}[t]{0.48\linewidth}
\centering
\begin{tabular}{|l|c|c|c|c|c|}
\hline
\textbf{Day} & \textbf{Hour} & \textbf{Req.} & \textbf{Action} & \textbf{C1/C2/C3/C4} & \textbf{Tot.} \\
\hline
Mon & 00:00 & C1 & Grant & 2/0/0/0 & 2 \\
Mon & 00:00 & C1 & Grant & 2/0/0/0 & 2 \\
Mon & 00:00 & C1 & \textbf{Rej.} & 2/0/0/0 & 2 \\
Mon & 01:00 & C1 & Grant & 2/0/0/0 & 2 \\
Mon & 01:00 & C1 & Grant & 2/0/0/0 & 2 \\
Mon & 01:00 & C1 & \textbf{Rej.} & 2/0/0/0 & 2 \\
Mon & 02:00 & C1 & Grant & 1/0/0/0 & 1 \\
Mon & 03:00 & C1 & Grant & 2/0/0/0 & 2 \\
Mon & 03:00 & C1 & Grant & 2/0/0/0 & 2 \\
Mon & 04:00 & C1 & Grant & 2/0/0/0 & 2 \\
Mon & 04:00 & C1 & Grant & 2/0/0/0 & 2 \\
Mon & 04:00 & C1 & \textbf{Rej.} & 2/0/0/0 & 2 \\
Mon & 05:00 & C1 & Grant & 2/0/0/0 & 2 \\
Mon & 05:00 & C1 & Grant & 2/0/0/0 & 2 \\
Mon & 08:00 & C2 & Grant & 0/2/0/1 & 3 \\
Mon & 08:00 & C2 & Grant & 0/2/0/1 & 3 \\
Mon & 08:00 & C4 & Grant & 0/2/0/1 & 3 \\
Mon & 09:00 & C2 & Grant & 0/1/0/2 & 3 \\
Mon & 09:00 & C4 & Grant & 0/1/0/2 & 3 \\
Mon & 09:00 & C4 & Grant & 0/1/0/2 & 3 \\
Mon & 10:00 & C3 & Grant & 0/0/2/0 & 2 \\
Mon & 10:00 & C3 & Grant & 0/0/2/0 & 2 \\
Mon & 11:00 & C4 & Grant & 0/0/0/1 & 1 \\
\hline
\end{tabular}
\caption{Round robin arbiter, hard 25\% reject}
\label{tab:gpu-allocation-roundrobin}
\end{subtable}
\hfill
\begin{subtable}[t]{0.48\linewidth}
\centering
\begin{tabular}{|l|c|c|c|c|c|}
\hline
\textbf{Day} & \textbf{Hour} & \textbf{Req.} & \textbf{Action} & \textbf{C1/C2/C3/C4} & \textbf{Tot.} \\
\hline
Mon & 00:00 & C1 & Grant & 2/0/0/0 & 2 \\
Mon & 00:00 & C1 & Grant & 2/0/0/0 & 2 \\
Mon & 00:00 & C1 & \textbf{Rej.} & 2/0/0/0 & 2 \\
Mon & 01:00 & C1 & Grant & 2/0/0/0 & 2 \\
Mon & 01:00 & C1 & Grant & 2/0/0/0 & 2 \\
Mon & 01:00 & C1 & Grant & 3/0/0/0 & 3 \\
Mon & 02:00 & C1 & Grant & 1/0/0/0 & 1 \\
Mon & 03:00 & C1 & Grant & 2/0/0/0 & 2 \\
Mon & 03:00 & C1 & Grant & 2/0/0/0 & 2 \\
Mon & 04:00 & C1 & Grant & 2/0/0/0 & 2 \\
Mon & 04:00 & C1 & Grant & 2/0/0/0 & 2 \\
Mon & 04:00 & C1 & Grant & 3/0/0/0 & 3 \\
Mon & 05:00 & C1 & Grant & 2/0/0/0 & 2 \\
Mon & 05:00 & C1 & Grant & 2/0/0/0 & 2 \\
Mon & 08:00 & C2 & Grant & 0/2/0/1 & 3 \\
Mon & 08:00 & C2 & Grant & 0/2/0/1 & 3 \\
Mon & 08:00 & C4 & Grant & 0/2/0/1 & 3 \\
Mon & 09:00 & C2 & Grant & 0/1/0/2 & 3 \\
Mon & 09:00 & C4 & Grant & 0/1/0/2 & 3 \\
Mon & 09:00 & C4 & Grant & 0/1/0/2 & 3 \\
Mon & 10:00 & C3 & Grant & 0/0/2/0 & 2 \\
Mon & 10:00 & C3 & Grant & 0/0/2/0 & 2 \\
Mon & 11:00 & C4 & Grant & 0/0/0/1 & 1 \\
\hline
\end{tabular}
\caption{Synthetic data of new desired policy}
\label{tab:gpu-allocation-synthetic}
\end{subtable}

\end{table}

Company A concludes from analyzing Table~\ref{tab:gpu-allocation-roundrobin} that their static fairness policy is costing them money. Between midnight and 5 am, C1 is the sole active customer, but the rigid 25\% cap rejects requests despite abundant unused capacity. They want to model how to dynamically shift the safety constraints with the underlying arbitration structure. 

Hybrid controllers address such problems by partitioning the task across architectural components. A neural module might handle complex pattern recognition, while a symbolic module enforces hard constraints like fairness or safety. This division exploits the complementary strengths of learning and formal methods. For instance, recent work has explored hybrid architectures that combine SSMs with traditional controllers, using queue-feedback mechanisms alongside SSMs that capture dependencies over the entire input history for resource provisioning on meshed web systems~\cite{lei2022state}.
These approaches yield a system whose components act independently of one another and model different components of the problem. Because there is no universal policy, this can create catastrophic edge cases~\cite{fremont2020formal}.

We propose an alternative approach. Rather than coupling independent modules, we train a single SSM to serve as the complete allocation policy. Company A can retroactively analyze their historical data and generate synthetic traces that reflect their desired behavior: granting requests when capacity is available, regardless of static caps, while still enforcing fairness when multiple clients compete. 

The data in Table~\ref{tab:gpu-allocation-synthetic} can then be used to train an SSM that captures both the arbitration logic and the more complex history-dependent dynamics. However, randomly initialized SSMs struggle to learn such systems efficiently; symbolic warm-starting dramatically improves both convergence and accuracy. We illustrate this point in Fig.~\ref{fig:RoundRobinArbitration_Policy} where we learn a dynamic arbitration policy which behaves like a round robin arbiter with 4 channels, but must constantly keep grants within an evolving safety bounds.
The impact on warm starting shows faster convergence to a stable accuracy rate, as well as a higher overall accuracy rate even after 1000 epochs of training.

\begin{figure}[t]
    \centering
    \includegraphics[width=0.9\linewidth]{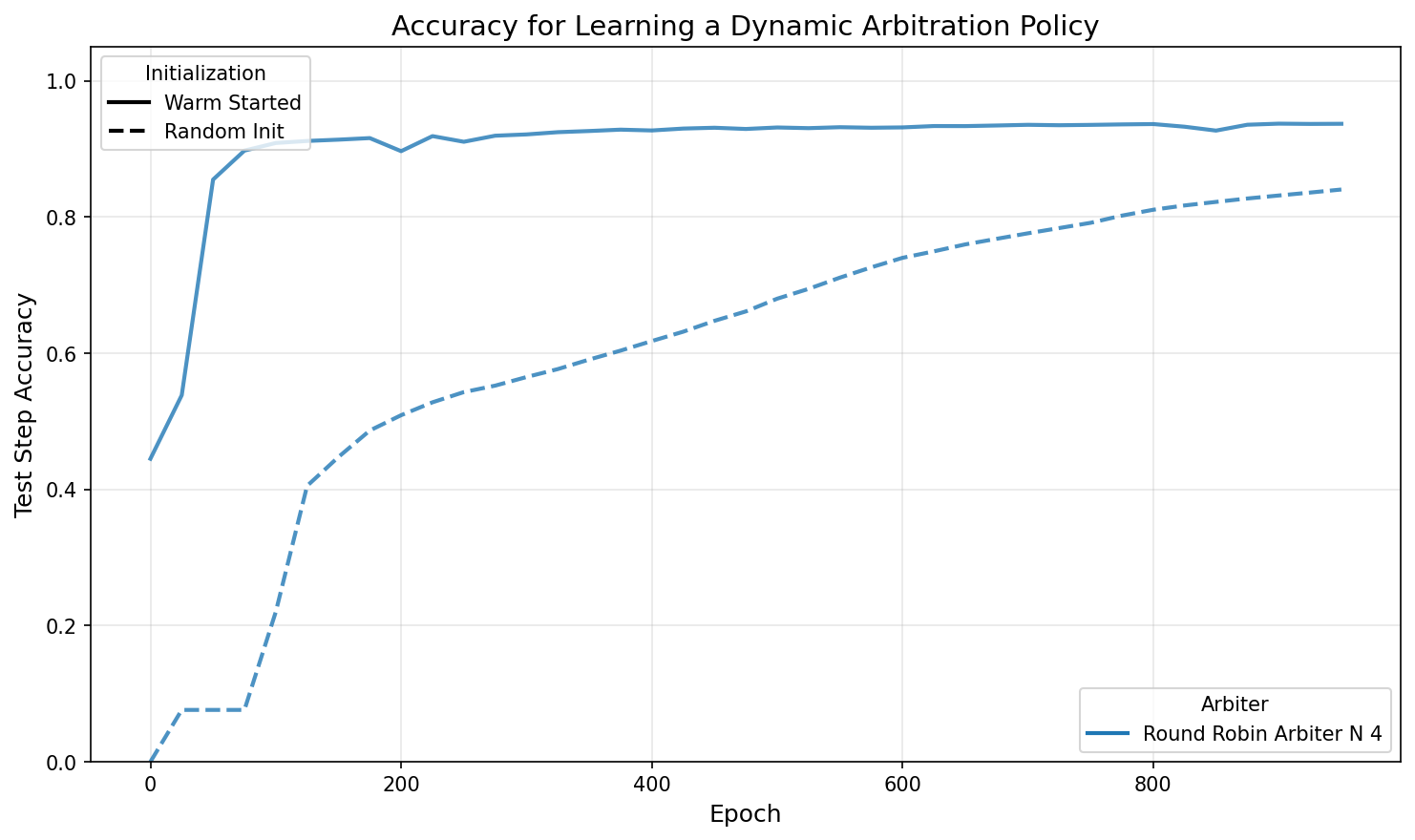}
    \caption{Warm-starting SSM using symbolic structure VS. learning using traditional random initialization on an dynamic round robin arbitration policy with 4 channels.}
    \label{fig:RoundRobinArbitration_Policy}
\end{figure}

\section{Preliminaries}
\subsection{State Space Models}
A classic SSM is represented as 
\begin{equation}
\label{eq:SSM}
\begin{aligned}
x(t+1) &= Ax(t) + B\mu(t) \\
y(t) &= Cx(t) +D\mu(t).
\end{aligned}
\end{equation}
whose components are
\begin{itemize}
    \item $x(t+1)$ is the future state
    \item $A$ is a matrix which describes how the states evolve over time 
    \item $x(t)$ is the current state
    \item $B$ is a matrix which describes how the current input effects the next state
    \item $\mu(t)$ is the current input
    \item $y(t)$ is the current output of the system
    \item $C$ is a matrix which describes how the current state effects the output
    \item $D$ describes how the current input effects the output
\end{itemize}
Often these systems are described more simply as
\begin{equation}
\label{eq:SSM_Moore}
\begin{aligned}
x(t+1) &= Ax(t) + B\mu(t) \\
y(t) &= Cx(t)
\end{aligned}
\end{equation}
where $D=0$~\cite{bourdois2024getonthessmtrain}. For modeling purposes, we use a discretized notion of time. These SSMs are linear in their input and their state. While recent SSM variants such as Mamba and S4~\cite{gu2024mamba} introduce additional complexity, we focus on the canonical form in Equation~\ref{eq:SSM_Moore}, as it more clearly reveals the fundamental properties of the architecture.
\subsection{Moore Machines}
A Moore machine is a system that describes a discrete symbolic state space whose future state is dependent on the current state and input, and output that is dependent on the current state. Specifically a Moore machine is defined as a system \\$A=(S,S_0,\Sigma,\Lambda,T,G)$ where 
\begin{itemize}
    \item $S$ is the set of finite states
    \item $S_0$ is the starting state
    \item $\Sigma$ is the input alphabet
    \item $\Lambda$ is the output alphabet
    \item $T$ is the transition functions mapping the input-state pair to future states\footnote{This is also commonly denoted as $\delta$.}
    \item $G$ is the output mapping which maps output-state pairs 
\end{itemize}
Specifically, a Moore Machine can be expressed as two distinct functions:
\begin{equation}
\label{eq:Moore_Machine}
\begin{aligned}
S_k &= F(S_i, \Sigma_j), \quad \text{where } S_0 \text{ is the initial state}, \\
\Lambda_j &= G(S_i).
\end{aligned}
\end{equation}
\section{Moore-SSMs}
While equations~\eqref{eq:SSM_Moore} and~\eqref{eq:Moore_Machine} are structurally similar, the additive property for generating new states in the SSM in equation~\eqref{eq:SSM_Moore} implies independence between the state and the input that does not appear in the Moore machine. 
\[
A x(t) + B \mu(t)
\]
The Moore machine in equation~\eqref{eq:Moore_Machine} returns the next state as a function of the current input and state jointly.
\[
S_k = F(S_i, \Sigma_j)
\]
Thus, to represent our Moore machine using SSM dynamics in Euclidean space, our mapping will need to allow the input and state to be able to operate independently. 

\begin{lemma}[Moore-SSMs]\label{lemma1:Moore_SSM_Equivalence}
Every Moore machine \( \mathcal{A} = (S, S_0, \Sigma, \Lambda, T, G) \), whose transition and output maps are defined in equation~\eqref{eq:Moore_Machine}, admits an equivalent representation as a discrete-time state--space model of the form
\[
\begin{aligned}
x(t+1) &= A x(t) + B \mu(t), \\
y(t) &= C x(t),
\end{aligned}
\]
where the state vector \(x(t)\) encodes the symbolic state \(S\), $\mu(t)$ encodes the input, and the matrices \(A\), \(B\), and \(C\) are structured so as to preserve the original input--output behavior and state transitions of the Moore machine.
\end{lemma}
\begin{proof}[lemma~\ref{lemma1:Moore_SSM_Equivalence}]
Given the Moore machine \( \mathcal{A} = (S, S_0, \Sigma, \Lambda, T, G) \) we first focus on representing the state updates of this system in the SSM form
    \[
    x(t+1)=Ax(t)+B\mu(t)
    \]
\textbf{Representing the Moore State Update Dynamics Using an SSM}    
\\
By treating each $S_i\in S$ as a basis vector for the orthonormal basis, we can create a trivial mapping of our symbolic states from a discrete space to a Euclidean space $E$ in $\mathbb{R}^N$ where $N=|S|$. Thus, every state vector denoted e$_i$ in the SSM is in E and is represented as 
\[
\mathbf{e}_i^\top = \begin{pmatrix} 0 & 0 & \cdots & 1 & \cdots & 0 \end{pmatrix}
\]
    The matrix $A$ shows how $x(t)=e_i$ evolves over time without any input to the system, yet in Moore machines, the state does not evolve without input; thus, we can set $A$ to the identity so $Ax(t)=Ix(t)=x(t)$.
    Setting $A=I$ yields the resulting equation
    \[
    x(t+1) - x(t)=e_j-e_i = B \mu(t)
    \]
    This formulation expresses the change in state as a function of the input. However, this representation does not eliminate the co-dependence of the state and input that characterizes a Moore machine. Although the difference \(x(t+1) - x(t)\) is linear in \(\mu(t)\), the underlying transition dynamics still depend jointly on the current symbolic state and the current input. Consequently, this formulation alone is insufficient to decouple the Moore machine’s state–input dependence.
    
    To preserve this joint dependence within an SSM framework, we reinterpret the input as a single composite symbol that encodes both the current state and the current input. Specifically, by defining the input space as the Kronecker product \(S \otimes \Sigma\), we construct an augmented input vector \(\mu(t)\) such that each input symbol uniquely represents a state–input pair. Under this encoding, the state transition becomes linear with respect to \(\mu(t)\), while still preserving the original Moore machine dynamics.
\[
\mu(t)^\top = \begin{bmatrix} \mu_{(s_1,\sigma_1)} & \mu_{(s_1,\sigma_2)} & \mu_{(s_1,\sigma_3)} & \mu_{(s_2,\sigma_1)} & \mu_{(s_2,\sigma_2)} & \mu_{(s_2,\sigma_3)} & \cdots & \mu_{(s_n,\sigma_3)} \end{bmatrix}
\]
By choosing \(\mu(t)\) of the above Kronecker-product form, we define a matrix\\ $B \in \mathbb{R}^{|S| \times (|S|\cdot|\Sigma|)}$ 
such that each column corresponds to a unique state--input pair \((s_i,\sigma_j)\). Specifically, the column indexed by \((i,j)\) is given by
\[
B_{:,(i,j)} = e_{T(i,j)} - e_i
\]
where \(x_i \in \mathbb{R}^{|S|}\) denotes the canonical basis vector corresponding to the current state \(s_i\), and \(e_{T(i,j)}\) denotes the canonical basis vector corresponding to the successor state determined by the transition function \(T(i,j)\) from $A$. Under this construction, the state update
\[
\begin{aligned}
    x(t+1) - x(t) = B \mu(t),\\
    x(t+1)  = x(t) + B \mu(t), \\
    x(t+1)  = Ax(t) + B \mu(t)
\end{aligned}
\]
exactly reproduces the Moore machine transition dynamics.
\\
\\
    \textbf{Representing the Moore Output Update Dynamics Using an SSM}
    \\
Encoding the output dynamics of the Moore machine \( \mathcal{A} \) into the SSM output equation \(y(t) = C x(t)\) is straightforward, since the outputs in a Moore machine depend solely on the current state. By representing each state \(s_k \in S\) as a one-hot vector \(x_t = e_k \in \mathbb{R}^{|S|}\), we may construct an output matrix $C \in \{0,1\}^{|\Lambda|* |S|}$
whose columns encode the Moore output function. Specifically, the \(k\)-th column of \(C\) is the one-hot vector corresponding to the output symbol \(\lambda = G(s_k)\). Under this construction, the linear readout
\[
y(t) = C x(t)
\]
exactly reproduces the Moore machine output, with \(y(t)\) given as the one-hot encoded representation of the output alphabet element \(\lambda \in \Lambda\). Thus, we have perfectly encoded $A$ in SSM dynamics. $\qed$
\end{proof}
\section{Learning Regular Languages from Data}
Lemma~\ref{lemma1:Moore_SSM_Equivalence} establishes exact equivalence between Moore machines and a class of SSMs. Yet, we show gradient descent (while training SSMs) fails to recover symbolic structure even on regular language tasks where such structure fully characterizes the target function. Automata learning methods, by contrast, exploit this structure directly. This disparity underscores the potential of symbolic priors to accelerate learning in continuous architectures. 
We evaluate randomly initialized SSMs against active and passive automata learning methods at learning automata structures synthesized from SYNTCOMP~\cite{SYNTCOMP2023}.
While leading automata learning tools such as AALpy use random automata structures for benchmarking~\cite{muvskardin2022aalpy}, a more structured benchmark set allows us more visibility into learning dynamics.
\begin{problem}~\label{problem:AutomataRecovery}(SYNTCOMP Emulation Task)
Specifically, our task is recovering a system $\mathcal{A}^*$ which generates infinite traces $\omega \in \Sigma \cup \Lambda$ where $\Sigma$ and $\Lambda$ are the input and output alphabet
whose behavior is accepted by the original system $\mathcal{A}$.
As we are testing SSMs that do not recover an automaton structure, we do not use $\omega$-testing and rather settle for a simpler task of emulating the language $L$. Our evaluation is the percentage of traces $\omega_i$ generated by $\mathcal{A}^*$ that are accepted by $\mathcal{A}$.
\end{problem}
\subsection{Methodology}
\begin{figure}[t]
    \centering
\begin{tikzpicture}[
    box/.style={rectangle, draw=black, semithick, minimum width=1.1cm, minimum height=0.6cm, align=center, fill=white, font=\fontsize{5}{5.5}\selectfont},
    orangebox/.style={rectangle, draw=orange!80!black, semithick, fill=orange!20, minimum width=1.4cm, minimum height=1cm, align=left, font=\fontsize{5}{5.5}\selectfont},
    lowerbox/.style={rectangle, draw=black, semithick, minimum height=1.3cm, align=center, fill=white},
    state/.style={circle, draw=black, semithick, minimum size=0.18cm, fill=white},
    bluestate/.style={circle, draw=blue!70!black, semithick, minimum size=0.14cm, fill=blue!30},
    graystate/.style={circle, draw=gray, semithick, minimum size=0.14cm, fill=gray!30},
    arrow/.style={-{Stealth[scale=0.5]}, semithick},
    dashedarrow/.style={-{Stealth[scale=0.5]}, semithick, dashed},
    tinyfont/.style={font=\fontsize{5}{5}\selectfont},
    microtext/.style={font=\fontsize{5}{3}\selectfont},
]

\node[orangebox] (tlsf) at (0,0) {
    \textbf{TLSF File}\\[0.5pt]
    inputs: $AP_1, AP_2, \cdots$\\
    outputs: $AP_{o1}, AP_{o2}, \cdots$\\
    Assumptions: $\varphi_{\text{A}}$\\
    Guarantees: $\varphi_{\text{G}}$
};

\node[box, right=1.3cm of tlsf, minimum width=0.7cm, minimum height=0.45cm] (dotfile) {Dot File};

\begin{scope}[shift={($(dotfile.east)+(0.25,0)$)}, scale=0.35]
    \node[state, minimum size=0.1cm] (dA) at (0,-.9cm) {};
    \node[state, minimum size=0.1cm] (dB) at (0.6,0.3) {};
    \node[state, minimum size=0.1cm] (dC) at (1.6,-.9) {};
    \draw[->, thin] (dA) -- (dB);
    \draw[->, thin] (dC) -- (dA);
    \draw[->, thin] (dB) to[bend left=15] (dC);
\end{scope}

\draw[arrow] (tlsf.east) -- (dotfile.west) node[midway, above, microtext, align=center, xshift=30pt,yshift=8pt] {Synthesis via Spot/LTLSynt};

\node[box, right=1.6cm of dotfile, minimum width=1cm, minimum height=0.5cm] (traces) {Generate\\Finite Traces};

\draw[arrow] ($(dotfile.east)+(0.7,0)$) -- (traces.west);

\node[right=0.8cm of traces, align=left, font=\fontsize{9}{5.5}\selectfont] (traceseq) {
    $\pi = \{ap_1, \neg ap_2; ap_1, ap_2; \cdots\}$
};

\draw[arrow] (traces.east) -- (traceseq.west) node[midway, above, microtext, align=center, xshift=23pt,yshift=8pt] {Prefix Closed Dataset};


\node[lowerbox, minimum width=2.5cm, minimum height=1.6cm] (lstar) at (-0.3,-2) {};
\node[tinyfont] at ($(lstar.south)+(0,-0.3)$) {\textbf{Active Learning L*}};

\begin{scope}[shift={($(lstar.center)+(0,0)$)}, scale=1.5]
    \node[rectangle, draw, minimum width=0.4cm, minimum height=0.22cm, fill=gray!10, microtext] (learner) at (-0.5, 0.12) {Learner};
    \node[rectangle, draw, minimum width=0.3cm, minimum height=0.4cm, fill=gray!10, microtext, align=center] (mat) at (0.5, 0) {$\in T$\\$\downarrow$};
    \node[microtext, align=center] at (0, 0.35) {Memb.\\Oracle};
    \node[microtext, align=center] at (-.5, -0.4) {Equiv.\\Oracle};
    \draw[->, very thin] (learner.east) -- ++(0.1,0) |- (mat.west);
    \draw[->, very thin] (mat.south) -- ++(0,-0.1) -| (learner.south);
    \node[microtext] at (.5, -0.4) {$M \in \mathcal{L}$};
\end{scope}

\node[lowerbox, minimum width=2.2cm, minimum height=1.6cm] (mealy) at (3.5,-2) {};
\node[tinyfont] at ($(mealy.south)+(0,-0.3)$) {\textbf{Passive Mealy Machine RPNI}};

\begin{scope}[shift={($(mealy.center)+(0,0)$)}, scale=0.45]
    \node[bluestate, minimum size=0.18cm] (mA) at (-1,-.5) {};
    
    \node[bluestate, minimum size=0.18cm] (mB) at (0.3, 0.6) {};
    \node[bluestate, minimum size=0.18cm] (mC) at (0.3, -0.6) {};
    
    \node[graystate, minimum size=0.18cm] (mD) at (1.5,-.5) {};
    
    \draw[->, thin, orange!80!black] (mA) -- (mB);
    \draw[->, thin, blue!70!black] (mA) -- (mC);
    
    \draw[->, thin, orange!80!black] (mB) -- (mD);
    \draw[->, thin, blue!70!black] (mC) -- (mD);
    
    \node[microtext] at (-1, -0.5) {$q_0$};
    \node[microtext] at (0.3, 0.6) {$q_1$};
    \node[microtext] at (0.3, -0.6) {$q_2$};
    \node[microtext] at (1.5, -0.5) {$q_3$};
\end{scope}
\node[lowerbox, minimum width=2.5cm, minimum height=1.6cm] (ssm) at (7.2,-2) {};
\node[tinyfont] at ($(ssm.south)+(0,-0.3)$) {\textbf{Autoregressive SSM}};

\begin{scope}[shift={($(ssm.center)+(0,0)$)}, scale=1]
    \node[circle, draw, fill=blue!20, minimum size=0.1cm] (iA) at (-0.7, 0.32) {};
    \node[circle, draw, fill=blue!20, minimum size=0.1cm] (iB) at (-0.7, 0) {};
    \node[circle, draw, fill=blue!20, minimum size=0.1cm] (iC) at (-0.7, -0.32) {};
    \node[circle, draw, fill=orange!30, minimum size=0.1cm] (hA) at (0, 0.44) {};
    \node[circle, draw, fill=orange!30, minimum size=0.1cm] (hB) at (0, 0.15) {};
    \node[circle, draw, fill=orange!30, minimum size=0.1cm] (hC) at (0, -0.15) {};
    \node[circle, draw, fill=orange!30, minimum size=0.1cm] (hD) at (0, -0.44) {};
    \node[circle, draw, fill=green!30, minimum size=0.1cm] (oA) at (0.7, 0.22) {};
    \node[circle, draw, fill=green!30, minimum size=0.1cm] (oB) at (0.7, -0.22) {};
    \foreach \i in {A,B,C} {
        \foreach \h in {A,B,C,D} {
            \draw[very thin, gray!50] (i\i) -- (h\h);
        }
    }
    \foreach \h in {A,B,C,D} {
        \foreach \o in {A,B} {
            \draw[very thin, gray!50] (h\h) -- (o\o);
        }
    }
    \node[microtext] at (-0.7, -0.6) {$x(t)$};
    \node[microtext] at (0.7, -0.6) {$y(t)$};
    \draw[->, very thin] (0.85, 0) to[out=22, in=-22] (0.85, 0.48) to[out=158, in=22] (-0.85, 0.48) to[out=-158, in=158] (-0.85, 0);
\end{scope}

\draw[arrow, rounded corners=4pt] (tlsf.south) -- ++(0,-0.45) -| (lstar.north);
\draw[arrow] (traceseq.south) -- ++(0,-0.32) -| (mealy.north);
\draw[arrow] (traceseq.south) -- ++(0,-0.32) -| (ssm.north);

\end{tikzpicture}
    \caption{Highlevel training pipeline visualization for learning regular languages from SYNTCOMP}
    \label{fig:TrainingPipeling}
\end{figure}
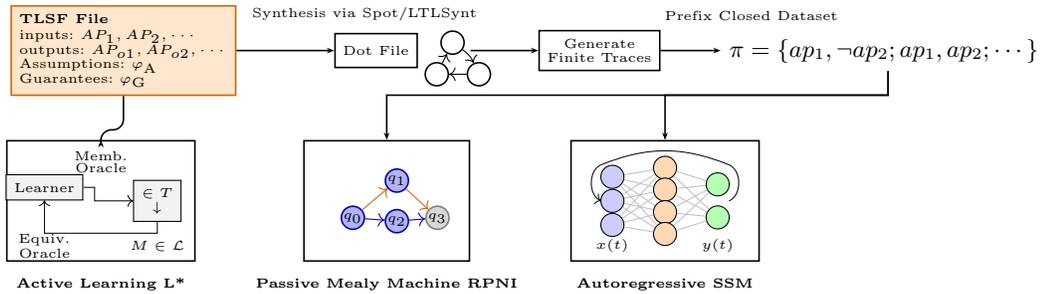
Fig.~\ref{fig:TrainingPipeling} contrasts three approaches for learning a regular language $L$ from a automata $\mathcal{A}$. Firstly we synthesize a system from the TLSF file using the \texttt{ltlsynt} tool~\cite{renkin2022dissecting} via Spot 2.14~\cite{duret2022spot}. This system is the oracle system for our active learning setup. Active learning generates data internally via queries, but gradient-based learning and passive learning require training datasets. We must therefore generate data from the synthesized systems that respects their assumptions.
\vspace{0.5em}
\newline
\noindent\textbf{Data Generation}
Data generation from automata systems is an underexplored task. There are two problems in generating data from these systems: (1) the data generated needs to be finite, but the original systems accept infinite behaviors, and (2) most data generation algorithms do not follow the complex assumptions of the underlying systems. 
We solve problem 1 by simplifying our requirements for generating training data. We leverage Spot semantics to do so. Our data generated needs to follow the rules of the underlying language for the length $T$ of the generated finite trace. We then append \texttt{cycle\{1\}}. This specifies that all behavior onward is true, thus any incorrect behavior of the trace would be in the first $T$ time steps. We validate our traces by generating the data using Spot semantics and then check the validity of this finite trace by using \texttt{autfilt} on an HOA~\cite{babiak2015hanoi} version of this system. (See Appendix~\ref{App:Sample_Traces_and_Logs} for sample trace).

The HOA file format represents a system that is equivalent to its Moore structure, which is deterministic in the union of the input and output pairs.
While tools such as HOAX~\cite{di2025execution} enable the generation of finite and infinite traces that respect the behavior of these systems using the HOA file format, these tools are slow and fail when trying to model complex assumptions. 

\tikzset{
    state/.style={circle, draw, minimum size=1cm, font=\small},
    active/.style={fill=orange!40},
    inactive/.style={fill=orange!15},
    every edge/.style={draw, ->, >=Stealth},
    input label/.style={font=\small, red!70!black},
}
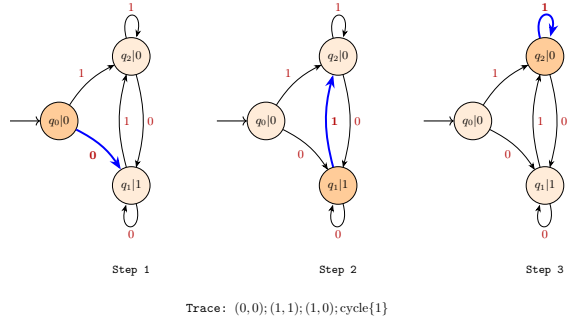
\begin{figure}
    \centering
\scalebox{0.5}{
\begin{tikzpicture}[node distance=1.5cm]

\begin{scope}[local bounding box=step1]
    \node[state, active] (q0) {$q_0 | 0$};
    \node[state, inactive, above right=1cm and 1.2cm of q0] (q2) {$q_2 | 0$};
    \node[state, inactive, below right=1cm and 1.2cm of q0] (q1) {$q_1 | 1$};
    
    \path (q0) edge[bend left=15] node[input label, above left] {1} (q2);
    \path (q0) edge[bend left=15, line width=1.5pt, blue] node[input label, below left] {\textbf{0}} (q1);
    \path (q2) edge[loop above] node[input label, above] {1} ();
    \path (q2) edge[bend left=15] node[input label, right] {0} (q1);
    \path (q1) edge[bend left=15] node[input label, right] {1} (q2);
    \path (q1) edge[loop below] node[input label, below] {0} ();
    
    \draw[->, thick] ([xshift=-0.8cm]q0.west) -- (q0.west);
    
    \node[below=1.5cm of q1, font=\footnotesize\ttfamily] {Step 1};
\end{scope}

\begin{scope}[xshift=5.5cm, local bounding box=step2]
    \node[state, inactive] (q0) {$q_0 | 0$};
    \node[state, inactive, above right=1cm and 1.2cm of q0] (q2) {$q_2 | 0$};
    \node[state, active, below right=1cm and 1.2cm of q0] (q1) {$q_1 | 1$};
    
    \path (q0) edge[bend left=15] node[input label, above left] {1} (q2);
    \path (q0) edge[bend left=15] node[input label, below left] {0} (q1);
    \path (q2) edge[loop above] node[input label, above] {1} ();
    \path (q2) edge[bend left=15] node[input label, right] {0} (q1);
    \path (q1) edge[bend left=15, line width=1.5pt, blue] node[input label, right] {\textbf{1}} (q2);
    \path (q1) edge[loop below] node[input label, below] {0} ();
    
    \draw[->, thick] ([xshift=-0.8cm]q0.west) -- (q0.west);
    
    \node[below=1.5cm of q1, font=\footnotesize\ttfamily] {Step 2};
\end{scope}

\begin{scope}[xshift=11cm, local bounding box=step3]
    \node[state, inactive] (q0) {$q_0 | 0$};
    \node[state, active, above right=1cm and 1.2cm of q0] (q2) {$q_2 | 0$};
    \node[state, inactive, below right=1cm and 1.2cm of q0] (q1) {$q_1 | 1$};
    
    \path (q0) edge[bend left=15] node[input label, above left] {1} (q2);
    \path (q0) edge[bend left=15] node[input label, below left] {0} (q1);
    \path (q2) edge[loop above, line width=1.5pt, blue] node[input label, above] {\textbf{1}} ();
    \path (q2) edge[bend left=15] node[input label, right] {0} (q1);
    \path (q1) edge[bend left=15] node[input label, right] {1} (q2);
    \path (q1) edge[loop below] node[input label, below] {0} ();
    
    \draw[->, thick] ([xshift=-0.8cm]q0.west) -- (q0.west);
    
    \node[below=1.5cm of q1, font=\footnotesize\ttfamily] {Step 3};
\end{scope}

\node[below=0.5cm of step2, font=\normalsize\ttfamily, align=center] (trace) {
    Trace: $(0,0);(1,1);(1,0);\textrm{cycle}\{1\}$
};

\end{tikzpicture}
}
   \caption{Methodology for generating assumption following traces by randomly picking transitions from states, generating assumption abiding traces}
    \label{fig:Data_Gen_Random_Walk}
\end{figure}
We use the GraphViz Dot file representation~\cite{gansner2009drawing} to encode our systems as deterministic with respect to inputs, as illustrated in Fig.~\ref{fig:Data_Gen_Random_Walk}. We generate data by performing random walks over this graph and logging the resulting traces, appending \texttt{cycle\{1\}} to each trace. By exploiting the ground-truth Moore machine structure, we ensure all data satisfies the specification assumptions, while random walks enable highly efficient large-scale data generation.
\vspace{0.5em}
\newline
\noindent\textbf{Active Automata Learning}
We perform active learning using AALpy to recover Mealy machines equivalent to our target Moore machines; AALpy has better support for Mealy machines, so we rely on the well-known equivalence of Moore machines and Mealy machines~\cite{klimovich2010transformation}. We follow a standard active learning setup using the L$^*$ algorithm. While there have been many advances in automata learning (see cf. Sec. ~\ref{sec:related}) L$^*$ is sufficient for illustrating the sample efficiency of active learning in this setting.
Our system under learning (SUL) is the originally synthesized automaton from the TLSF specification.
For membership queries, L$^*$ submits an input sequence to the SUL, which simulates the automaton and returns the corresponding sequence of output symbols—each output being a boolean valuation over the output propositions. This provides the teacher's response without requiring access to the automaton's internal structure. For equivalence queries, we use a random walk algorithm, which samples input sequences to search for counterexamples between the hypothesis and the SUL. 
This setup is simple, and the choice of SUL as the original system does not reflect real-world applications. However, it creates a useful comparison to other learning methods since all data comes from a clean source with no noise.

We run three trials of active learning on each SYNTCOMP benchmark and report the average number of queries as the sample complexity. For each recovered automaton, we generate 1000 traces of length 20 using our data generation pipeline and compute the percentage accepted by the original system. The reported accuracy is the average across the three trials. Trials were run on AWS Lambda; those that timed out are marked as failed and excluded from the final results. (See Appendix~\ref{App:Sample_Traces_and_Logs} for the output structure using automata learning). 
\vspace{0.5em}
\newline
\noindent\textbf{Passive Automata Learning}
The passive learning system uses RPNI for learning Mealy machines equivalent to the underlying Moore machine. We generate traces from our data generation pipeline via random walks across the Dot file representation, and convert those to a prefix-closed representation to learn an automaton from data. Fig.~\ref{fig:Prefix_Closed_Data} shows the converted data generated from dot files into a form acceptable for the passive learning setup. 
\begin{figure}[t]
    \centering
    \begin{lstlisting}
#Original Trace
g_0&!c_0&r_0&!g_1&!c_1&r_1;!g_0&c_0&!r_0&g_1&!c_1&!r_1;cycle{1}
    
#Prefix-closed representation
( !c_0 & r_0 & !c_1 & r_1) ,(g_0 & !g_1 )
( !c_0 & r_0 & !c_1 & r_1 ; c_0  & !r_0 & !c_1 & !r_1),(!g_0 & g_1 )
    \end{lstlisting}
    \caption{Prefix-closed representation of sample trace}
    \label{fig:Prefix_Closed_Data}
\end{figure}

Unlike the active learning setup, passive learning needs data provided beforehand. We evaluate our passive learning setup on a range of sample sizes from 5000-30000 traces of length 20. We measure accuracy identically to the active learning setup: we generate 1000 traces from the learned system and compute the percentage accepted by the original system. We report the sample complexity as the number of samples required to reach maximum accuracy. Systems that achieve 100\% accuracy do not continue to larger sample sizes. 
\subsubsection{SSM Learning}~\label{subsec:SSM_FSM}
We learn these regular languages using state-space models (SSMs) from  Equation~\ref{eq:SSM_Moore}.
Models are trained autoregressively to predict the output at time $T$ given the input sequence from times $0$ through $T$, using the PyTorch library~\cite{imambi2021pytorch}. 
Fig.~\ref{fig:SSM_FSM_Visualization} illustrates the SSM architecture used to learn the SYNTCOMP benchmarks.

\begin{figure}[t]
    \centering
\scalebox{.45}{
\begin{tikzpicture}[
    >=Stealth,
    node distance=1.5cm and 2cm,
    box/.style={draw, rounded corners, minimum width=1.8cm, minimum height=0.9cm, align=center, font=\small},
    param/.style={box, fill=blue!15},
    op/.style={box, fill=orange!20},
    io/.style={box, fill=green!15},
    hidden/.style={circle, draw, fill=gray!20, minimum size=1cm, font=\small},
    arrow/.style={->, thick},
    label/.style={font=\footnotesize\itshape, text=gray}
]

\node[io] (mut) {$\mu(t)$};
\node[label, below=0.1cm of mut] {input};

\node[param, right=1.5cm of mut] (E) {$E$};
\node[label, above=0.1cm of E] {embed};

\node[op, right=1.2cm of E] (relu) {ReLU};

\node[hidden, right=1.2cm of relu] (muembed) {$\tilde{\mu}_t$};
\node[label, below=0.1cm of muembed] {embedded};

\node[param, right=1.2cm of muembed] (B) {$B$};

\node[hidden, above=3cm of B, xshift=-5cm] (xprev) {$x(t)$};
\node[label, left=.5cm of xprev] {prev state};

\node[param, right=1.2cm of xprev] (A) {$A$};

\node[op, right=1.2cm of B] (sum) {$+$};

\node[op, right=1.2cm of sum] (tanh) {tanh};

\node[hidden, right=1.2cm of tanh] (xt) {$x(t+1)$};
\node[label, above=0.1cm of xt] {state};

\node[param, right=1.5cm of xt] (C) {$C$};
\node[label, above=0.1cm of C] {readout};

\node[io, right=1.5cm of C] (yt) {$y(t+1)$};
\node[label, below=0.1cm of yt] {output};

\draw[arrow] (mut) -- (E);
\draw[arrow] (E) -- (relu);
\draw[arrow] (relu) -- (muembed);
\draw[arrow] (muembed) -- (B);
\draw[arrow] (B) -- (sum) node[midway, below, font=\scriptsize] {$B\tilde{\mu}_t$};

\draw[arrow] (xprev) -- (A);
\draw[arrow] (A) -- (A -| sum.north) -- (sum) node[pos=0.7, right, font=\scriptsize] {$Ax(t)$};

\draw[arrow] (sum) -- (tanh);
\draw[arrow] (tanh) -- (xt);
\draw[arrow] (xt) -- (C);
\draw[arrow] (C) -- (yt);

\draw[arrow, dashed, gray] (xt.north) -- ++(0, 0.8) -| ($(xprev.north) + (0, 0.8)$) -- (xprev.north);
\node[label, above=1.3cm of xt] {recurrence};

\node[label, left=0.5cm of xprev, align=right,yshift=.5cm] (x0label) {$x_0$ \\ (learned)};
\draw[->, gray, dashed] (x0label) -- (xprev);

\begin{scope}[on background layer]
    \node[fit=(mut)(E)(relu)(muembed), draw=green!50!black, dashed, rounded corners, inner sep=8pt, label={[font=\footnotesize]above:Input Embedding}] {};
    
    \node[fit=(B)(A)(sum)(tanh)(xprev), draw=blue!50!black, dashed, rounded corners, inner sep=8pt, label={[font=\footnotesize]above:State Transition}] {};
    
    \node[fit=(C)(yt), draw=orange!50!black, dashed, rounded corners, inner sep=8pt, label={[font=\footnotesize]above:Moore Output}] {};
\end{scope}

\end{tikzpicture}
}
    \caption{SSM architecture used to learn SYNTCOMP regular languages}
    \label{fig:SSM_FSM_Visualization}
\end{figure}
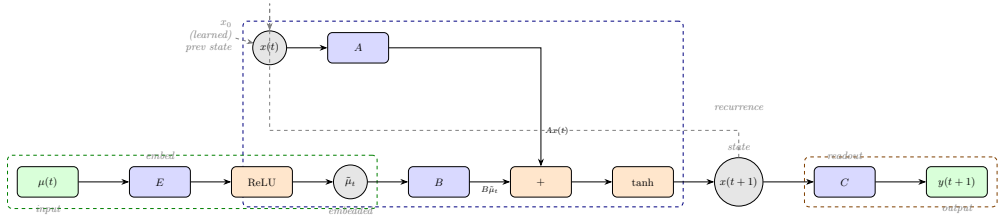

Training is performed using binary cross-entropy loss to predict a boolean output tensor of dimensionality $\Lambda \times 1$, where $\Lambda$ denotes the number of output APs in the original system. 
During training, the model produces a real-valued output vector in $\mathbb{R}^{\Lambda \times 1}$, which is passed through a logit function to obtain a Boolean prediction in $\{0,1\}^{\Lambda \times 1}$. Binary cross-entropy loss is then computed between this prediction and the ground-truth Boolean outputs. We use a standard learning rate $.001$ and the Adam Optimizer~\cite{adam2014method}.

Inputs at each time step are represented as Boolean vectors in $\{0,1\}^{\Sigma \times 1}$, where $\Sigma$ denotes the number of input APs in the system. 
These inputs are embedded into a $32$-dimensional continuous space using a learned linear embedding layer followed by a ReLU activation. 
This embedding procedure is standard practice and allows the SSM to operate over continuous representations while preserving the discrete structure of the input alphabet.
To compute the output at time $T$, the model is provided with the full input sequence as a matrix in $\{0,1\}^{\Sigma \times T}$.
The SSM is unrolled over time, iteratively computing the latent state at each step using the state update equation in Equation~\ref{eq:SSM_Moore}. The final latent state is then mapped through the output function to produce the output vector.

Data are generated using the same procedure as in passive learning. We generate 10000 traces of length 20, of which $10\%$ are held out for testing. Although models are trained autoregressively, evaluation is performed using a stricter notion of trace acceptance rather than per-step prediction accuracy.
Concretely, given an input sequence from time $0$ to $20$, we require that the model correctly predict \emph{every} output along the trace. A test trace is considered accepted if the predicted output sequence exactly matches the ground-truth output sequence produced by the original machine. We report the percentage of test traces accepted under this criterion, which measures how well the SSM emulates the full input--output behavior of the target system.
Unlike passive and active learning methods, gradient descent can iteratively optimize itself. We account for this in our sample complexity calculations by reporting the epoch of highest accuracy $\times$ 9000, which represents the number of times the training data was used to learn the regular language; for consistency we train all tasks on a max 1000 epochs.

Lastly, we emphasize that in this experiment, the matrices $A$, $B$, and $C$ are randomly initialized as is common practice. 
Initializing these matrices according to lemma~\ref{lemma1:Moore_SSM_Equivalence} would yield a system that is already perfectly aligned with the target Moore machine and would therefore trivialize the learning task.
The purpose of this experiment is instead to study the sample efficiency gains obtained by explicitly leveraging symbolic structure in the underlying problem, and to assess whether such structure can be naturally discovered through gradient-based optimization.

\subsection{Sample Efficiency of Symbolic vs. Gradient-Based Learning}~\label{subsec:sample_efficiency}
\begin{figure}[t]
    \centering
    \includegraphics[width=1\linewidth]{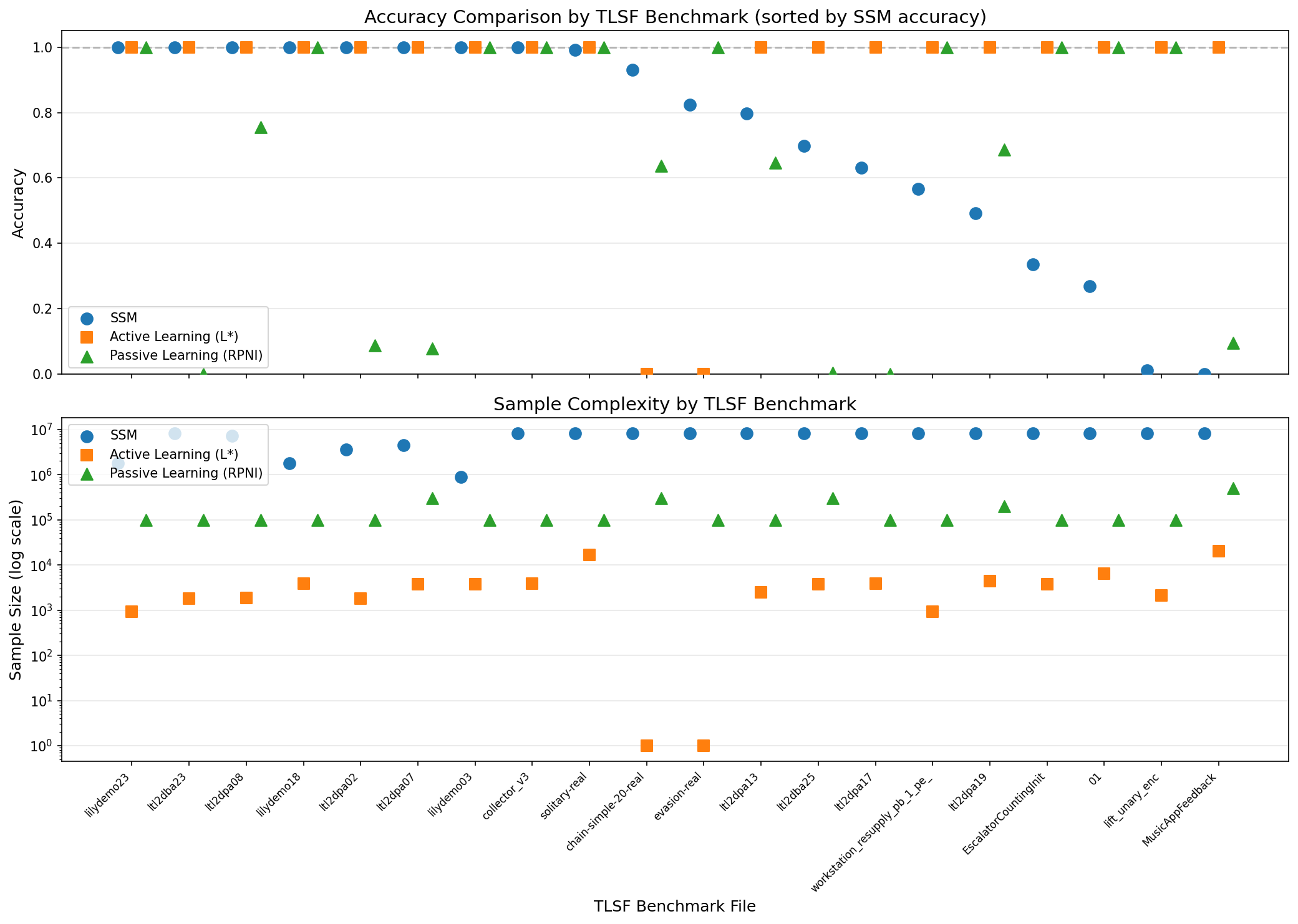}
    \caption{Accuracy and sample complexity (Log Scale) for learning a random subset of 20 SYNTCOMP examples}
    \label{fig:Sample_Complexity_Accuracuy}
\end{figure}
Fig.~\ref{fig:Sample_Complexity_Accuracuy} presents accuracy and sample complexity on a suite of 20 SYNTCOMP benchmarks learned using active, passive, and gradient-based methods. The full set of evaluated benchmarks is provided in a table in Appendix~\ref{app:SyntcompBenchmarks}. The first chart reports accuracy on the trace acceptance task for each learning method on its corresponding benchmark, while the second chart reports sample complexity on a logarithmic scale. Table~\ref{tab:perfect_emulation} shows the total percentages of perfectly emulated systems learned by each of the learning methods. This table shows that overall SSMs perform much worse than passive or active learning methods in learning these regular language problems.
\begin{table}[t]
    \centering
    \begin{tabular}{|l|c|}
        \hline
        Learning Method & Perfect Emulation \\
        \hline
        Gradient-based SSMs & 33.3\% \\
        \hline
        L$^*$ & 77.3\% \\
        \hline
        RPNI & 56.0\% \\
        \hline
    \end{tabular}
    
    \caption{Percentage of SYNTCOMP benchmarks solved with 100\% accuracy by each learning method.}
    \label{tab:perfect_emulation}
\end{table}

\subsection{Symbolic Structure in Gradient-Based Models}
Although Moore machines admit exact realizations as SSMs, gradient-based training of SSMs on tasks that exhibit symbolic structure does not recover representations close to the symbolic structures. 
Instead, these models learn to perfectly emulate the input-output dynamics of the system without inducing the underlying symbolic organization.
\begin{figure}[t]
    \centering
    \includegraphics[width=.7\linewidth]{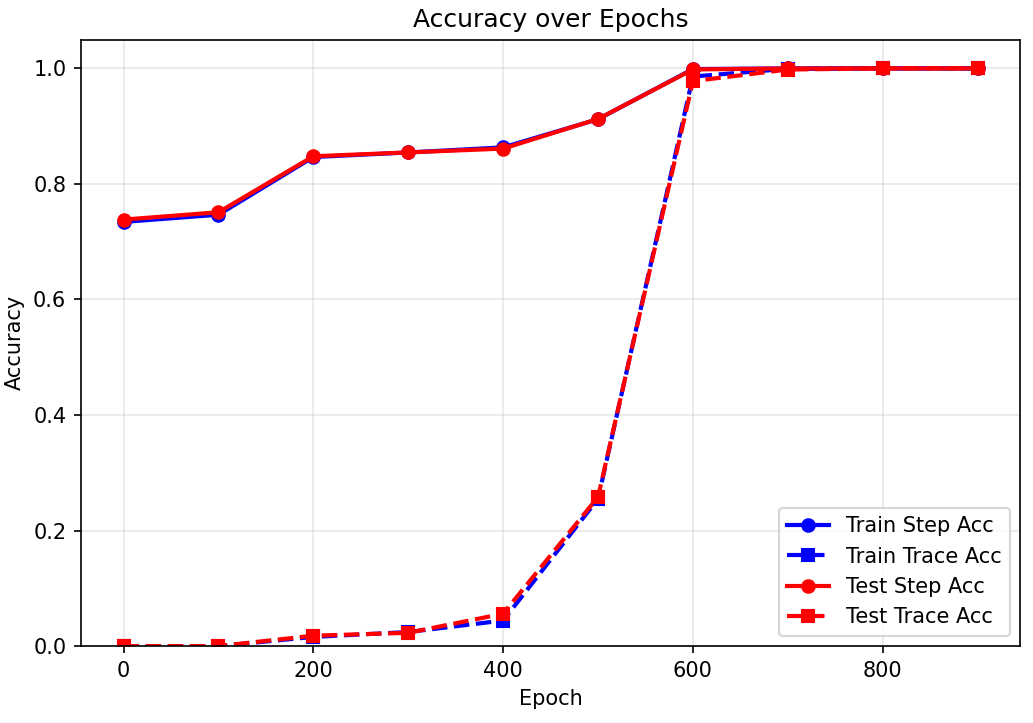}
    \caption{Training and Testing Accuracy on Predicting Next Step and Trace Acceptance for collector/parametric/collector\_v2, N=3}
    \label{fig:collector_v2_n_3_accuracy}
\end{figure}
Fig.~\ref{fig:collector_v2_n_3_accuracy} shows that an SSM trained to learn the collector\_v2 regular language achieves 100\% accuracy on Trace Acceptance after 700 epochs for the test data, demonstrating that this model can perfectly simulate the input-output relationships of the model. 
However, the states that are generated by the SSM as it processes various sequences of inputs do not exhibit any symbolic organization in Euclidean space.
\begin{figure}[t]
    \centering
    \makebox[\textwidth][c]{\includegraphics[width=1\textwidth]{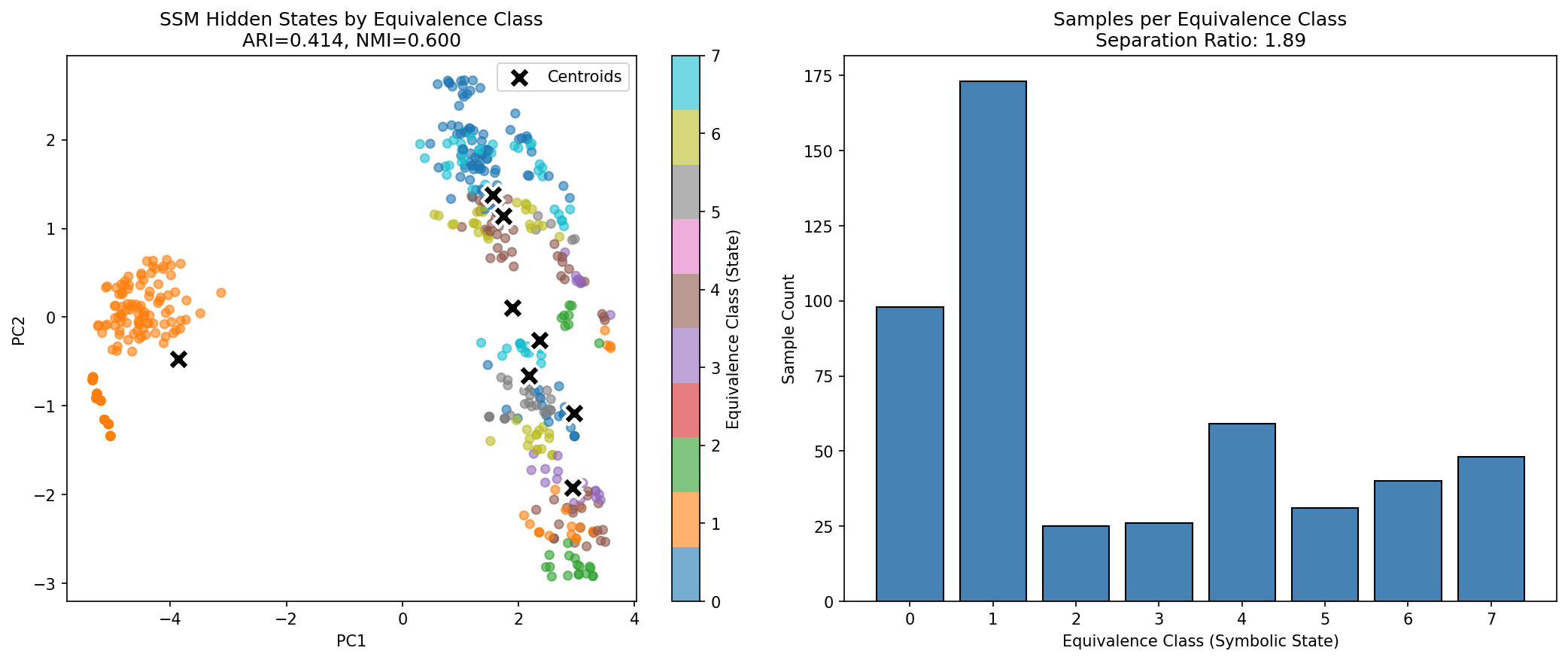}}
    \caption{PCA projection of SSM hidden states after processing 1,000 random length-5 traces, colored by ground-truth automaton state. Counts indicate traces terminating in each state.}
    \label{fig:Clusters}
\end{figure}

Fig.~\ref{fig:Clusters} visualizes the latent states of an SSM trained to emulate the parametric \texttt{collector\_v2}.
We collect the discrete state and corresponding latent state from the SSM for 1000 legal traces from the $\textit{collector\_v2}$. We cluster these states based on their first two principal components and color them based on their discrete state. We then map each of the centroids for the ground truth states. 
The overlap between ground-truth states in this projection shows that the SSM does not preserve the discrete structure of the target automaton.
We quantify cluster quality using three metrics: the separation ratio, Adjusted Rand Index (ARI), and Normalized Mutual Information (NMI). The separation ratio measures the ratio of the variance of elements between different classes and the variance of elements in the same class. Here, a ratio of 1.89 suggests only weak separation. ARI measures agreement between the predicted clustering and ground-truth labels, adjusted for chance—a value of 0 corresponds to random assignment, while 1 indicates perfect agreement. NMI similarly quantifies the mutual information between cluster assignments and true labels, normalized to
[0,1]. The observed values (ARI = 0.414, NMI = 0.600) indicate partial but imperfect correspondence: the SSM captures some structure, as evidenced by the non-zero ARI, but exhibits substantial overlap between states that should be distinct.

 These results suggest that gradient descent on sequence prediction objectives does not recover discrete automaton structure in any robust sense. Nevertheless, our results in Sec.~\ref{subsec:sample_efficiency} demonstrate that symbolic structure is critical for efficient learning of regular languages. Moreover, lemma~\ref{lemma1:Moore_SSM_Equivalence} provides a mechanism to encode this structure directly: by initializing SSMs with the transition dynamics of a Moore machine, we obtain models that inherit symbolic inductive biases while retaining the representational flexibility of continuous state spaces. We leverage these two insights to learn complex systems that lie beyond the reach of classical automata learning—particularly those requiring unbounded memory or robustness to noise.

\section{Warm-Starting SSMs Using Symbolic Structure}
We revisit the motivating example (Sec.~\ref{sec:motiv}) of learning dynamic arbitration policies for cloud resource allocation.
We demonstrate that leveraging symbolic initialization via Lemma~\ref{lemma1:Moore_SSM_Equivalence} enables efficient learning of dynamic arbitration policies for cloud resource allocation.
Specifically, we define our problem as:
\begin{problem}[Dynamic Arbitration]~\label{DynamicArbitration}
Given a finite-state arbiter $A$, we define a system $A^*$ that respects the dynamics of $A$ by default. However, $A^*$ must keep track of the history of the number of grants given out to $g_1,g_2,\cdots,g_n$ as well as $|g|/n$,  where $|g|$ represents the total grants given since timestep 0. While $A^*$ defaults to behaving as $A$, it has an existing safety constraint that no $g_i$ can ever receive more than $(|g|/n)+k$ grants for some constant $k$. We evaluate with $K=3$, which provides an optimally noisy setting.
\end{problem}
Problem~\ref{DynamicArbitration} is not a tractable problem for passive and active automata learning, as this problem requires infinite memory. SSMs can leverage recurrence, similar to RNNs, to solve these problems despite being trained on finite-length traces~\cite{siegelmann1992computational}. We show that symbolic warm-starting is a way of gaining better accuracy and faster learning using SSMs on this problem.
We evaluate our approach on five families of arbiters from the SYNTCOMP benchmarks~\cite{SYNTCOMP2023}:
\begin{itemize}
    \item \texttt{arbiter\_zoo/parametric/arbiter.tlsf}
    \item \texttt{arbiter\_zoo/parametric/arbiter\_with\_cancel.tlsf}
    \item \texttt{full\_arbiter/parametric/full\_arbiter.tlsf}
    \item \texttt{round\_robin\_arbiter/parametric/round\_robin\_arbiter.tlsf}
    \item \texttt{prioritized\_arbiter/parametric/prioritized\_arbiter.tlsf}
\end{itemize}
Each arbiter is evaluated across multiple parameterizations, with
$n \in \{2,\dots,6\}$.
\begin{algorithm}[t]~\label{Algorithim1}
\caption{Symbolic Initialization of an SSM from a Moore Machine}
\label{alg:ssm_init}
\KwIn{Moore machine $\mathcal{A}=(S,S_0,\Sigma,\Lambda,T,G)$}
\KwOut{SSM matrices $A,B,C$}

\textbf{Step 1 (Noise scale).}
Choose $\varepsilon \in (0,1)$ and add Gaussian noise
$\mathcal{N}(0,\varepsilon)$. This noise scale is added to the zero values in the matrix to make learning smoother while preserving a structure close to the symbolic encoding.

\textbf{Step 2 (State matrix).}
Construct $I \in \mathbb{R}^{|S|\times|S|}$
\[
A \leftarrow I + \mathcal{N}(0,\varepsilon),\text{ where $A_{ij}=0$}
\]

\textbf{Step 3 (Input matrix).}
Construct $B \in \mathbb{R}^{|S|\times(|S|\cdot|\Sigma|)}$ as in the proof:
for each $(s_i,\sigma_j)$ with $T(s_i,\sigma_j)=s_{k}$,
set the corresponding column to $e_{k}-e_i$.
Add noise:
\[
B_{ij} \leftarrow B_{ij} + \mathcal{N}(0,\varepsilon),\text{ where $B_{ij}=0$}
\]

\textbf{Step 4 (Output matrix).}
Construct $C \in \mathbb{R}^{|\Lambda|\times|S|}$ as in the proof:
for each state $s_i$ with output $\lambda_\ell=G(s_i)$, set $C_{\ell i}=1$.
Add noise:
\[
C_{ij} \leftarrow C_{ij} + \mathcal{N}(0,\varepsilon),\text{ where $C_{ij}=0$}
\]

\Return $(A,B,C)$
\end{algorithm}
\begin{figure}[t]
    \centering
    \includegraphics[width=0.8\linewidth]{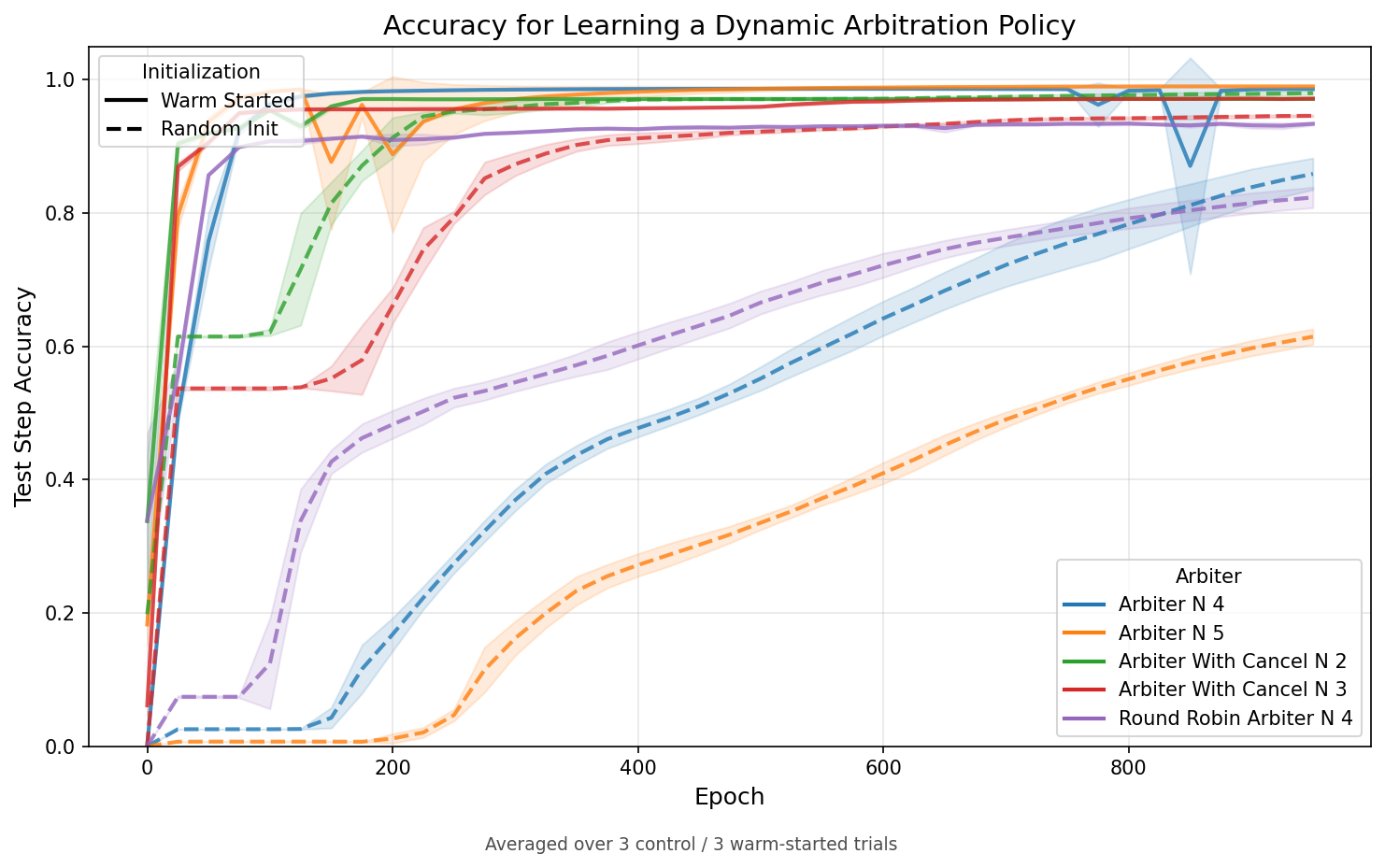}
    \caption{Learning and performance of SSMs learning dynamic arbitration policies over various arbiter settings. Warm-start VS. random initialization. N denotes parameterization}
    \label{fig:AdaptiveArbitrationResults}
\end{figure}
\subsection{Methodology}
Our overall neural architecture follows that shown in Fig.~\ref{fig:SSM_FSM_Visualization}. However, the SSM layer is initialized using Algorithm~\ref{alg:ssm_init}, and we omit a linear input embedding layer to preserve the symbolic structure of the inputs.
Unlike the SSM architecture described in Sec.~\ref{subsec:SSM_FSM}, the input and outputs are represented as one-hot encoded vectors over 2$^\Sigma\times1$ and 2$^\Lambda\times1$, respectively. Furthermore, the  input passed into the SSM is represented as the Kronecker product of the state and one-hot encoded input 
\[
\begin{aligned}
    \mu_i=S_i\otimes \alpha_i, \text{ with $S_0$ being the initial state}
\end{aligned}
\]
breaking the codependence of state and input in equation~\ref{eq:Moore_Machine} and allowing us to fully invoke lemma~\ref{lemma1:Moore_SSM_Equivalence} and construct our Moore-SSMs. 
We generate 10000 traces using our original data-generation procedure.
Next, we iterate over these traces and convert them to conform to the rules of the system defined in Problem~\ref{DynamicArbitration}.
We note that for every system, the newly generated data is not accepted by the original automaton, further demonstrating that our new system is fundamentally different than the initialized system.
Furthermore, initial accuracy for this task on warm-started systems is 
10\%-50\%. If the new problem were fundamentally the same as the old problem, our initializations would not require any learning to perform this task.
We train on 9000 traces and evaluate on the auto-regressive task of predicting the next output $\lambda$ for our test set.

\subsection{Warm-Starting Evaluation}
Symbolic warm-starting enables massive efficiency gains in learning these tasks over the various benchmarks. 
We define successful convergence as achieving at least 90\% test accuracy.
This threshold is intentionally stringent while remaining attainable for all models.
Across all tasks, Moore-SSMs reach this benchmark on average 243 epochs earlier than randomly initialized SSMs.
To quantify statistical significance, we apply a Mann-Whitney $U$ test, which is appropriate for non-parametric convergence time distributions.
The resulting $p$-value of $0.0122$ indicates a statistically significant difference in convergence speed of these respective systems. 

Fig.~\ref{fig:AdaptiveArbitrationResults} shows the results of learning this task using 5 different arbitration schemes over 3 trial runs. 
(See Appendix~\ref{app:adaptive_arbitration_results} for full results). 
In all of these cases, warm-starting systems can learn more efficiently and gain higher or equivalent accuracy. Warm-starting is specifically well-suited for learning systems with larger alphabets.
For example, the system derived from  \texttt{arbiter\_zoo/parametric\\/arbiter.tlsf} with 5 requests and 5 grants achieves close to perfect accuracy after 300 epochs with warm-starting. Using random initialization, the model achieves only 60\% accuracy after 950 epochs on the same problem. 

While symbolic warm-starting is effective on these tasks, it introduces a significant increase in model dimensionality, which can lead to memory constraints.
All experiments were conducted on an NVIDIA RTX~3080 GPU with 10GB of VRAM.
For larger problem instances, specifically the round-robin arbiter with parameterizations $n=5$ and $n=6$, symbolic warm-starting occasionally failed due to GPU memory exhaustion.
While leveraging larger GPUs might rectify this problem, scaling to a larger system for warm-starting might yield unforeseen memory issues. We postulate that using spectral-learning on these Moore-SSMs might yield close systems in smaller dimensionality, but leave these questions to future work.
\section{Related Work}~\label{sec:related}
\noindent \textbf{Warm-starting} Warm-starting is a widely used technique in machine learning for improving learning efficiency by initializing a model with a representation learned from a closely related task. This idea has been particularly successful in reinforcement learning, where bootstrapping an initial system can significantly accelerate policy optimization~\cite{tio2025machine}. 
Warm-starting has become central to LLM training, where learning largely proceeds through self-supervised fine-tuning~\cite{rani2023self}. However, prior work has not explored how symbolic structures can be used to initialize learning in complex settings.
Our work functions as the first to explore the effect of leveraging this automata structure explicitly and demonstrates the effectiveness that even simple Moore machines can have when learning more complex settings.
\vspace{0.5em}
\newline
\noindent\textbf{Automata-Neural Equivalence}
The relationship between neural networks and automata has been extensively studied through the formal connection between weighted finite automata (WFA) and recurrent neural networks (RNNs).
Prior work in this area has largely focused on extracting WFA representations from trained RNNs in order to uncover their underlying symbolic structure and to provide theoretical guarantees on model behavior and training~\cite{li2024connecting,wei2022extracting}.
Building on this direction, classical automata learning algorithms such as L$^*$ have been extended to the weighted setting, enabling the recovery of WFA representations directly from trained RNNs~\cite{2020weighted}.

Despite these advances, existing work has largely focused on relatively simple neural architectures and their correspondence with weighted automata, leaving open the question of how automata structures can be efficiently represented within more expressive neural models. To the best of our knowledge, no prior work has formally established a connection between Moore machines and SSMs.
\vspace{0.5em}
\newline
\noindent\textbf{Bisimulation}
The concept of behavioral closeness between automata has been formalized through bisimulation~\cite{sangiorgi2009origins}, which partitions states into equivalence classes that are indistinguishable under all possible future behaviors.
This strict notion of equivalence enables principled abstraction and has been extended to learning settings involving large and potentially infinite-state systems~\cite{abate2024bisimulation}. The robustness of bisimulation-based abstractions has also led to successful applications in reinforcement learning~\cite{shi2024self}. 
However, because bisimulation defines closeness through equivalence rather than degree, it remains brittle in machine learning contexts where approximate similarity is more natural. Our work addresses this limitation by embedding automata into Euclidean space as state space models, providing a continuous notion of closeness that is better suited to gradient-based learning.

\section{Conclusion}
This work establishes a formal equivalence between Moore machines and a class of state space models, enabling a continuous notion of behavioral closeness between discrete computational structures. By embedding automata into Euclidean space, we provide a framework where symbolic structure serves as an inductive bias for sequence learning. Our experiments demonstrate that symbolic warm-starting yields orders of magnitude greater sample efficiency than gradient-based training alone, and extends learning to systems requiring infinite memory. Future work should explore scaling these methods to understand how formal automata can be more deeply integrated with neural architectures.
\bibliographystyle{splncs04} 
\bibliography{refs}

@inproceedings{gu2024mamba,
  title={Mamba: Linear-time sequence modeling with selective state spaces},
  author={Gu, Albert and Dao, Tri},
  booktitle={First conference on language modeling},
  year={2024}
}

@dataset{SYNTCOMP2023,
  author       = {Jacobs, Swen and Perez, Guillermo A. and Schlehuber-Caissier, Philipp},
  title        = {Data, scripts, and results from SYNTCOMP 2023},
  year         = {2023},
  publisher    = {Zenodo},
  doi          = {10.5281/zenodo.8161423}
}

@article{hamilton1994state,
  title={State-space models},
  author={Hamilton, James D},
  journal={Handbook of econometrics},
  volume={4},
  pages={3039--3080},
  year={1994},
  publisher={Elsevier}
}

@article{angluin1987learning,
  title={Learning regular sets from queries and counterexamples},
  author={Angluin, Dana},
  journal={Information and computation},
  volume={75},
  number={2},
  pages={87--106},
  year={1987},
  publisher={Elsevier}
}

@article{tirnuaucua2012survey,
  title={A survey of state merging strategies for dfa identification in the limit},
  author={T{\^\i}rn{\u{a}}uc{\u{a}}, Cristina},
  journal={Triangle: llenguatge, literatura, computaci{\'o}},
  pages={121--136},
  year={2012}
}

@article{muvskardin2022aalpy,
  title={AALpy: an active automata learning library},
  author={Mu{\v{s}}kardin, Edi and Aichernig, Bernhard K and Pill, Ingo and Pferscher, Andrea and Tappler, Martin},
  journal={Innovations in Systems and Software Engineering},
  volume={18},
  number={3},
  pages={417--426},
  year={2022},
  publisher={Springer}
}

@article{li2024connecting,
  title={Connecting weighted automata, tensor networks and recurrent neural networks through spectral learning},
  author={Li, Tianyu and Precup, Doina and Rabusseau, Guillaume},
  journal={Machine Learning},
  volume={113},
  number={5},
  pages={2619--2653},
  year={2024},
  publisher={Springer}
}

@article{2020weighted,
  title={Weighted Automata Extraction from Recurrent Neural Networks via Regression on State Spaces},
  author={Okudono, Takamasa and Waga, Masaki and Sekiyama, Taro and Hasuo, Ichiro},
  journal={Proceedings of AAAI'20},
  year={2020}
}

@inproceedings{wei2022extracting,
  title={Extracting weighted finite automata from recurrent neural networks for natural languages},
  author={Wei, Zeming and Zhang, Xiyue and Sun, Meng},
  booktitle={International Conference on Formal Engineering Methods},
  pages={370--385},
  year={2022},
  organization={Springer}
}

@article{tio2025machine,
  title={From Machine to Human Learning: Towards Warm-Starting Teacher Algorithms with Reinforcement Learning Agents},
  journal={open review},
  author={Tio, Sidney and Li, Wenjun and Karunasena, Ramesha and Jimmy, Ho Tian Sheng and Varakantham, Pradeep},
  year={2025}
}

@misc{bourdois2024getonthessmtrain,
  author       = {Bourdois, Loïck},
  title        = {Introduction to State Space Models (SSM)},
  howpublished = {Hugging Face Blog},
  year         = {2024},
  month        = {July},
  day          = {19},
  url          = {https://huggingface.co/blog/lbourdois/get-on-the-ssm-train}
}

@inproceedings{di2025execution,
  title={Execution and Monitoring of HOA Automata with HOAX},
  author={Di Stefano, Luca},
  booktitle={International Conference on Runtime Verification},
  pages={44--53},
  year={2025},
  organization={Springer}
}

@inproceedings{duret2022spot,
  title={From spot 2.0 to spot 2.10: What’s new?},
  author={Duret-Lutz, Alexandre and Renault, Etienne and Colange, Maximilien and Renkin, Florian and Gbaguidi Aisse, Alexandre and Schlehuber-Caissier, Philipp and Medioni, Thomas and Martin, Antoine and Dubois, J{\'e}r{\^o}me and Gillard, Cl{\'e}ment and others},
  booktitle={International Conference on Computer Aided Verification},
  pages={174--187},
  year={2022},
  organization={Springer}
}

@article{renkin2022dissecting,
  title={Dissecting ltlsynt},
  author={Renkin, Florian and Schlehuber-Caissier, Philipp and Duret-Lutz, Alexandre and Pommellet, Adrien},
  journal={Formal Methods in System Design},
  volume={61},
  number={2},
  pages={248--289},
  year={2022},
  publisher={Springer}
}

@inproceedings{babiak2015hanoi,
  title={The Hanoi omega-automata format},
  author={Babiak, Tom{\'a}{\v{s}} and Blahoudek, Franti{\v{s}}ek and Duret-Lutz, Alexandre and Klein, Joachim and K{\v{r}}et{\'\i}nsk{\`y}, Jan and M{\"u}ller, David and Parker, David and Strej{\v{c}}ek, Jan},
  booktitle={International Conference on Computer Aided Verification},
  pages={479--486},
  year={2015},
  organization={Springer}
}

@article{gansner2009drawing,
  title={Drawing graphs with Graphviz},
  author={Gansner, Emden R},
  journal={Technical report, AT\&T Bell Laboratories, Murray, Tech. Rep, Tech. Rep.},
  year={2009}
}

@incollection{imambi2021pytorch,
  title={PyTorch},
  author={Imambi, Sagar and Prakash, Kolla Bhanu and Kanagachidambaresan, GR},
  booktitle={Programming with TensorFlow: solution for edge computing applications},
  pages={87--104},
  year={2021},
  publisher={Springer}
}

@article{aichernig2022active,
  title={Active vs. passive: A comparison of automata learning paradigms for network protocols},
  author={Aichernig, Bernhard K and Mu{\v{s}}kardin, Edi and Pferscher, Andrea},
  journal={arXiv preprint arXiv:2209.14031},
  year={2022}
}

@article{aichernig2024benchmarking,
  title={Benchmarking combinations of learning and testing algorithms for automata learning},
  author={Aichernig, Bernhard K and Tappler, Martin and Wallner, Felix},
  journal={Formal Aspects of Computing},
  volume={36},
  number={1},
  pages={1--37},
  year={2024},
  publisher={ACM New York, NY}
}

@inproceedings{boubaker2016formal,
  title={Formal verification of cloud resource allocation in business processes using event-b},
  author={Boubaker, Souha and Mammar, Amel and Graiet, Mohamed and Gaaloul, Walid},
  booktitle={2016 IEEE 30th International Conference on Advanced Information Networking and Applications (AINA)},
  pages={746--753},
  year={2016},
  organization={IEEE}
}

@article{fan2016formal,
  title={A formal aspect-oriented method for modeling and analyzing adaptive resource scheduling in cloud computing},
  author={Fan, Guisheng and Yu, Huiqun and Chen, Liqiong},
  journal={IEEE Transactions on Network and Service Management},
  volume={13},
  number={2},
  pages={281--294},
  year={2016},
  publisher={IEEE}
}

@inproceedings{garfatta2018formal,
  title={Formal modelling and verification of cloud resource allocation in business processes},
  author={Garfatta, Ikram and Klai, Kais and Graiet, Mohamed and Gaaloul, Walid},
  booktitle={OTM Confederated International Conferences" On the Move to Meaningful Internet Systems"},
  pages={552--567},
  year={2018},
  organization={Springer}
}

@article{almarhabi2024distributed,
  title={Distributed arbiter a lightweight, and enhanced access control mechanism for cloud computing},
  author={Almarhabi, Ali Khalid},
  journal={Thermal Science},
  volume={28},
  number={6 Part B},
  pages={4969--4977},
  year={2024}
}

@inproceedings{sharara2021recurrent,
  title={A recurrent neural network based approach for coordinating radio and computing resources allocation in cloud-ran},
  author={Sharara, Mahdi and Hoteit, Sahar and V{\`e}que, V{\'e}ronique},
  booktitle={2021 IEEE 22nd International Conference on High Performance Switching and Routing (HPSR)},
  pages={1--7},
  year={2021},
  organization={IEEE}
}

@article{zi2024time,
  title={Time-Series Load Prediction for Cloud Resource Allocation Using Recurrent Neural Networks},
  author={Zi, Yun},
  journal={Journal of Computer Technology and Software},
  volume={3},
  number={7},
  year={2024}
}

@article{lei2022state,
  title={State space model and queuing network based Cloud Resource Provisioning for Meshed Web Systems},
  author={Lei, Yamin and Cai, Zhicheng and Li, Xiaoping and Buyya, Rajkumar},
  journal={IEEE Transactions on Parallel and Distributed Systems},
  volume={33},
  number={12},
  pages={3787--3799},
  year={2022},
  publisher={IEEE}
}

@article{klimovich2010transformation,
  title={Transformation of a mealy finite-state machine into a moore finite-state machine by splitting internal states},
  author={Klimovich, AS and Solov’ev, VV},
  journal={Journal of Computer and Systems Sciences International},
  volume={49},
  number={6},
  pages={900--908},
  year={2010},
  publisher={Springer}
}

@article{adam2014method,
  title={A method for stochastic optimization},
  author={Adam, Kingma DP Ba J and others},
  journal={arXiv preprint arXiv:1412.6980},
  volume={1412},
  number={6},
  year={2014}
}

@inproceedings{siegelmann1992computational,
  title={On the computational power of neural nets},
  author={Siegelmann, Hava T and Sontag, Eduardo D},
  booktitle={Proceedings of the fifth annual workshop on Computational learning theory},
  pages={440--449},
  year={1992}
}

@article{peng2021dl2,
  title={DL2: A deep learning-driven scheduler for deep learning clusters},
  author={Peng, Yanghua and Bao, Yixin and Chen, Yangrui and Wu, Chuan and Meng, Chen and Lin, Wei},
  journal={IEEE Transactions on Parallel and Distributed Systems},
  volume={32},
  number={8},
  pages={1947--1960},
  year={2021},
  publisher={IEEE}
}

@article{rani2023self,
  title={Self-supervised learning: A succinct review},
  author={Rani, Veenu and Nabi, Syed Tufael and Kumar, Munish and Mittal, Ajay and Kumar, Krishan},
  journal={Archives of Computational Methods in Engineering},
  volume={30},
  number={4},
  pages={2761--2775},
  year={2023},
  publisher={Springer}
}

@article{sangiorgi2009origins,
  title={On the origins of bisimulation and coinduction},
  author={Sangiorgi, Davide},
  journal={ACM Transactions on Programming Languages and Systems (TOPLAS)},
  volume={31},
  number={4},
  pages={1--41},
  year={2009},
  publisher={ACM New York, NY, USA}
}

@inproceedings{abate2024bisimulation,
  title={Bisimulation learning},
  author={Abate, Alessandro and Giacobbe, Mirco and Schnitzer, Yannik},
  booktitle={International Conference on Computer Aided Verification},
  pages={161--183},
  year={2024},
  organization={Springer}
}

@inproceedings{shi2024self,
  title={Self-Supervised Bisimulation Action Chunk Representation for Efficient RL},
  author={Shi, Lei and Jianye, HAO and Tang, Hongyao and Dong, Zibin and Zheng, Yan},
  booktitle={Neurips Safe Generative AI Workshop 2024},
  year={2024}
}

@inproceedings{fremont2020formal,
  title={Formal scenario-based testing of autonomous vehicles: From simulation to the real world},
  author={Fremont, Daniel J and Kim, Edward and Pant, Yash Vardhan and Seshia, Sanjit A and Acharya, Atul and Bruso, Xantha and Wells, Paul and Lemke, Steve and Lu, Qiang and Mehta, Shalin},
  booktitle={2020 IEEE 23rd International Conference on Intelligent Transportation Systems (ITSC)},
  pages={1--8},
  year={2020},
  organization={IEEE}
}
\appendix
\section{Figures}
\subsection{Sample Traces and Logs}~\label{App:Sample_Traces_and_Logs}
\begin{figure}
    \centering
    \begin{lstlisting}
    g_0&!c_0&r_0&!g_1&!c_1&r_1;!g_0&c_0&!r_0&g_1&!c_1&!r_1;cycle{1}
    \end{lstlisting}
    \caption{Sample Trace From Parametric Arbiter With Cancellations Using Spot Semantics and \texttt{cycle\{1\}}}
    \label{fig:Sample_Data}
\end{figure}
\begin{figure}
    \centering
\begin{lstlisting}
    {"tlsf_file": "benchmarks/tlsf_2/sweap/arbiter-with-failure-real.tlsf",
    "trial": 2,"sample_size": 154482, "accuracy": 100.0, "status": "success"}
\end{lstlisting}
\caption{Sample output from one trial of active automata learning}
    \label{fig:active_automata_learning_trial_1}
\end{figure}

\begin{figure}[h]
    \centering
    \begin{lstlisting}
    {"tlsf_file": "benchmarks/tlsf/tsl_paper/KitchenTimerV3.tlsf", "trial": 1,
    "trace_length": 20, "results": 
    [{"num_traces": 5000, "accuracy": 11.2}, 
    {"num_traces": 10000, "accuracy": 26.0},
    {"num_traces": 15000, "accuracy": 28.6},
    {"num_traces": 20000, "accuracy": 13.9},
    {"num_traces": 25000, "accuracy": 35.6},
    {"num_traces": 30000, "accuracy": 93.5}], 
    "final_accuracy": 93.5, "final_traces": 30000, "status": "success"}
    \end{lstlisting}
    \caption{Output of Passive Learning Trial Run on tsl\_paper/KitchenTimerV3 From SYNTCOMP}
    \label{fig:PassiveLearningOutput}
\end{figure}

\begin{figure}[h]
    \centering
    \begin{lstlisting}
    {"file": "/tlsf/parametric/abcg_arbiter.tlsf",
    "training_samples": 10000,
    "converged_epoch": 400,
    "test_trace_acc": 1.0,
    "epoch_history": [
    {"epoch": 0, "loss": 0.6949,  
    "test_step_acc": 0.1643, "test_trace_acc": 0.0},
    {"epoch": 100, "loss": 0.4375,
    "test_step_acc": 0.3726, "test_trace_acc": 0.0},
    {"epoch": 200, "loss": 0.4125,
    "test_step_acc": 0.3726, "test_trace_acc": 0.0},
    {"epoch": 300, "loss": 0.1361,
    "test_step_acc": 0.985, "test_trace_acc": 0.733},
    {"epoch": 400, "loss": 0.0411,
    "test_step_acc": 1.0, "test_trace_acc": 1.0 }]}
    \end{lstlisting}
    
    \caption{SSM Trained on parametric/abcg\_arbiter Training Run}
    \label{fig:trainingrun_abcg}
\end{figure}

\subsection{SYNTCOMP Benchmark Results}~\label{app:SyntcompBenchmarks}
\tiny
\begin{longtable}{p{5cm}|rrr|rrr}
\caption{Comparison of learning methods on SYNTCOMP benchmarks.} \label{tab:benchmark_comparison} \\
\toprule
Benchmark & SSM Acc (\%) & L* Acc (\%) & RPNI Acc (\%) & SSM Samples & L* Samples & RPNI Samples \\
\midrule
\endfirsthead

\multicolumn{7}{c}{\tablename\ \thetable{} -- continued from previous page} \\
\toprule
Benchmark & SSM Acc (\%) & L* Acc (\%) & RPNI Acc (\%) & SSM Samples & L* Samples & RPNI Samples \\
\midrule
\endhead

\midrule
\multicolumn{7}{r}{Continued on next page} \\
\endfoot

\bottomrule
\endlastfoot

$\mathtt{01}$ & 26.9 & 100.0 & 100.0 & 8,100,000 & 6,346 & 5,000 \\
$\mathtt{02}$ & 7.6 & 100.0 & 100.0 & 8,100,000 & 6,440 & 5,000 \\
$\mathtt{03}$ & 2.3 & 99.0 & 0.0 & 8,100,000 & 79,764 & 5,000 \\
$\mathtt{EscalatorSmart}$ & 0.0 & 100.0 & 100.0 & 8,100,000 & 6,959 & 5,000 \\
$\mathtt{Gamemodule}$ & 0.2 & 100.0 & 100.0 & 8,100,000 & 4,721 & 5,000 \\
$\mathtt{KitchenTimerV1}$ & 100.0 & 100.0 & 100.0 & 4,500,000 & 1,878 & 5,000 \\
$\mathtt{KitchenTimerV2}$ & 39.4 & 100.0 & 100.0 & 8,100,000 & 7,843 & 5,000 \\
$\mathtt{KitchenTimerV3}$ & 1.1 & 100.0 & 93.5 & 8,100,000 & 11,775 & 30,000 \\
$\mathtt{KitchenTimerV5}$ & 0.0 & 100.0 & 12.4 & 8,100,000 & 23,099 & 10,000 \\
$\mathtt{MusicAppFeedback}$ & 0.0 & 100.0 & 9.6 & 8,100,000 & 20,296 & 25,000 \\
$\mathtt{MusicAppSimple}$ & 3.0 & 100.0 & 100.0 & 8,100,000 & 5,048 & 5,000 \\
$\mathtt{SPIReadSdi}$ & 0.0 & 100.0 & 100.0 & 8,100,000 & 4,759 & 5,000 \\
$\mathtt{SPIWriteManag}$ & 10.1 & 0.0 & 100.0 & 8,100,000 & 6,533 & 5,000 \\
$\mathtt{SPIWriteSdi}$ & 0.3 & 0.0 & 100.0 & 8,100,000 & 8,624 & 5,000 \\
$\mathtt{SliderDefault}$ & 51.2 & 100.0 & 0.0 & 8,100,000 & 1,974 & 5,000 \\
$\mathtt{SliderScored}$ & 0.0 & 100.0 & 100.0 & 8,100,000 & 6,267 & 5,000 \\
$\mathtt{TorcsGearing}$ & 100.0 & 100.0 & 100.0 & 1,800,000 & 912 & 5,000 \\
$\mathtt{abcg\_arbiter}$ & 100.0 & 100.0 & 100.0 & 3,600,000 & 921 & 5,000 \\
$\mathtt{arbiter}$ & 100.0 & 100.0 & 100.0 & 1,800,000 & 888 & 5,000 \\
$\mathtt{box-limited-real}$ & 0.6 & 100.0 & 96.6 & 8,100,000 & 6,563 & 20,000 \\
$\mathtt{box-real}$ & 30.3 & 100.0 & 100.0 & 8,100,000 & 291,353 & 5,000 \\
$\mathtt{chomp}$ & 1.5 & 100.0 & 100.0 & 8,100,000 & 6,044 & 5,000 \\
$\mathtt{collector\_v1}$ & 81.8 & 100.0 & 0.3 & 8,100,000 & 4,484 & 25,000 \\
$\mathtt{collector\_v2}$ & 100.0 & 100.0 & 6.5 & 8,100,000 & 4,345 & 30,000 \\
$\mathtt{collector\_v3}$ & 100.0 & 100.0 & 100.0 & 8,100,000 & 3,877 & 5,000 \\
$\mathtt{diagonal-real}$ & 7.4 & 100.0 & 98.2 & 8,100,000 & 53,711 & 10,000 \\
$\mathtt{evasion-real}$ & 82.3 & 0.0 & 100.0 & 8,100,000 & — & 5,000 \\
$\mathtt{example10a}$ & 100.0 & 100.0 & 100.0 & 4,500,000 & 1,046 & 5,000 \\
$\mathtt{example10}$ & 100.0 & 100.0 & 6.7 & 4,500,000 & 959 & 15,000 \\
$\mathtt{example7a}$ & 100.0 & 100.0 & 100.0 & 4,500,000 & 1,883 & 5,000 \\
$\mathtt{example7}$ & 100.0 & 100.0 & 100.0 & 900,000 & 1,857 & 5,000 \\
$\mathtt{example8}$ & 100.0 & 100.0 & 100.0 & 1,800,000 & 1,870 & 5,000 \\
$\mathtt{example9}$ & 100.0 & 100.0 & 100.0 & 2,700,000 & 925 & 5,000 \\
$\mathtt{f-real-real}$ & 11.2 & 0.0 & 100.0 & 8,100,000 & — & 5,000 \\
$\mathtt{full\_arbiter}$ & 100.0 & 100.0 & 100.0 & 3,600,000 & 1,912 & 5,000 \\
$\mathtt{g-real-real}$ & 88.1 & 0.0 & 100.0 & 8,100,000 & — & 15,000 \\
$\mathtt{gf-real-real}$ & 17.4 & 94.9 & 2.7 & 8,100,000 & 72,393 & 15,000 \\
$\mathtt{heim-buechi-real}$ & 98.9 & 0.0 & 100.0 & 8,100,000 & — & 5,000 \\
$\mathtt{lift\_gr1+}$ & 0.0 & 100.0 & 44.2 & 8,100,000 & 19,565 & 30,000 \\
$\mathtt{lift\_gr1}$ & 2.4 & 100.0 & 100.0 & 8,100,000 & 4,758 & 10,000 \\
$\mathtt{lift\_unary\_enc}$ & 1.1 & 100.0 & 100.0 & 8,100,000 & 2,135 & 5,000 \\
$\mathtt{lift}$ & 0.1 & 100.0 & 100.0 & 8,100,000 & 1,976 & 5,000 \\
$\mathtt{lilydemo03}$ & 100.0 & 100.0 & 100.0 & 900,000 & 3,768 & 5,000 \\
$\mathtt{lilydemo04}$ & 100.0 & 100.0 & 39.1 & 2,700,000 & 3,879 & 25,000 \\
$\mathtt{lilydemo05}$ & 100.0 & 100.0 & 100.0 & 3,600,000 & 3,850 & 5,000 \\
$\mathtt{lilydemo06}$ & 100.0 & 100.0 & 100.0 & 6,300,000 & 4,008 & 5,000 \\
$\mathtt{lilydemo07}$ & 100.0 & 100.0 & 100.0 & 4,500,000 & 3,931 & 5,000 \\
$\mathtt{lilydemo08}$ & 100.0 & 100.0 & 100.0 & 900,000 & 905 & 5,000 \\
$\mathtt{lilydemo09}$ & 100.0 & 100.0 & 100.0 & 6,300,000 & 950 & 5,000 \\
$\mathtt{lilydemo10}$ & 100.0 & 100.0 & 100.0 & 900,000 & 1,829 & 5,000 \\
$\mathtt{lilydemo14}$ & 100.0 & 100.0 & 100.0 & 900,000 & 1,844 & 5,000 \\
$\mathtt{lilydemo15}$ & 100.0 & 100.0 & 100.0 & 8,100,000 & 1,923 & 5,000 \\
$\mathtt{lilydemo16}$ & 60.2 & 100.0 & 100.0 & 8,100,000 & 5,455 & 5,000 \\
$\mathtt{lilydemo17}$ & 100.0 & 100.0 & 100.0 & 1,800,000 & 1,857 & 5,000 \\
$\mathtt{lilydemo18}$ & 100.0 & 100.0 & 100.0 & 1,800,000 & 3,880 & 5,000 \\
$\mathtt{lilydemo19}$ & 100.0 & 100.0 & 100.0 & 2,700,000 & 1,871 & 5,000 \\
$\mathtt{lilydemo20}$ & 75.3 & 100.0 & 100.0 & 8,100,000 & 1,833 & 5,000 \\
$\mathtt{lilydemo21}$ & 0.6 & 99.4 & 100.0 & 8,100,000 & 25,482 & 10,000 \\
$\mathtt{lilydemo22}$ & 40.6 & 100.0 & 100.0 & 8,100,000 & 5,026 & 5,000 \\
$\mathtt{lilydemo23}$ & 100.0 & 100.0 & 100.0 & 1,800,000 & 913 & 5,000 \\
$\mathtt{load\_balancer}$ & 74.2 & 100.0 & 100.0 & 8,100,000 & 4,020 & 5,000 \\
$\mathtt{ltl2dba01}$ & 72.3 & 100.0 & 36.4 & 8,100,000 & 3,818 & 10,000 \\
$\mathtt{ltl2dba02}$ & 54.8 & 100.0 & 48.5 & 8,100,000 & 5,515 & 30,000 \\
$\mathtt{ltl2dba03}$ & 98.2 & 100.0 & 21.1 & 8,100,000 & 3,902 & 5,000 \\
$\mathtt{ltl2dba04}$ & 99.5 & 100.0 & 0.5 & 8,100,000 & 3,802 & 5,000 \\
$\mathtt{ltl2dba05}$ & 63.6 & 100.0 & 15.6 & 8,100,000 & 6,040 & 20,000 \\
$\mathtt{ltl2dba06}$ & 19.6 & 100.0 & 13.6 & 8,100,000 & 5,942 & 15,000 \\
$\mathtt{ltl2dba07}$ & 48.2 & 100.0 & 14.9 & 8,100,000 & 156,904 & 15,000 \\
$\mathtt{ltl2dba10}$ & 100.0 & 100.0 & 100.0 & 2,700,000 & 1,852 & 5,000 \\
$\mathtt{ltl2dba12}$ & 84.2 & 100.0 & 0.0 & 8,100,000 & 1,921 & 5,000 \\
$\mathtt{ltl2dba13}$ & 100.0 & 100.0 & 100.0 & 4,500,000 & 3,730 & 5,000 \\
$\mathtt{ltl2dba14}$ & 78.8 & 100.0 & 0.5 & 8,100,000 & 3,797 & 10,000 \\
$\mathtt{ltl2dba16}$ & 31.5 & 100.0 & 100.0 & 8,100,000 & 1,821 & 5,000 \\
$\mathtt{ltl2dba18}$ & 23.1 & 100.0 & 11.3 & 8,100,000 & 3,889 & 10,000 \\
$\mathtt{ltl2dba19}$ & 100.0 & 100.0 & 6.5 & 6,300,000 & 2,090 & 5,000 \\
$\mathtt{ltl2dba20}$ & 7.6 & 100.0 & 8.4 & 8,100,000 & 5,989 & 20,000 \\
$\mathtt{ltl2dba22}$ & 100.0 & 100.0 & 0.0 & 4,500,000 & 915 & 5,000 \\
$\mathtt{ltl2dba23}$ & 100.0 & 100.0 & 0.0 & 8,100,000 & 1,824 & 5,000 \\
$\mathtt{ltl2dba25}$ & 69.9 & 100.0 & 0.4 & 8,100,000 & 3,800 & 15,000 \\
$\mathtt{ltl2dba26}$ & 58.4 & 100.0 & 8.8 & 8,100,000 & 3,772 & 20,000 \\
$\mathtt{ltl2dba\_E}$ & 69.1 & 100.0 & 68.1 & 8,100,000 & 1,912 & 5,000 \\
$\mathtt{ltl2dba\_Q}$ & 67.2 & 100.0 & 90.3 & 8,100,000 & 2,011 & 10,000 \\
$\mathtt{ltl2dba\_U1}$ & 100.0 & 100.0 & 53.2 & 2,700,000 & 1,859 & 25,000 \\
$\mathtt{ltl2dba\_alpha}$ & 100.0 & 100.0 & 100.0 & 3,600,000 & 1,870 & 5,000 \\
$\mathtt{ltl2dba\_beta}$ & 62.5 & 100.0 & 52.4 & 8,100,000 & 7,631 & 10,000 \\
$\mathtt{ltl2dpa01}$ & 100.0 & 100.0 & 100.0 & 2,700,000 & 934 & 5,000 \\
$\mathtt{ltl2dpa02}$ & 100.0 & 100.0 & 8.7 & 3,600,000 & 1,814 & 5,000 \\
$\mathtt{ltl2dpa03}$ & 25.7 & 100.0 & 100.0 & 8,100,000 & 12,887 & 10,000 \\
$\mathtt{ltl2dpa04}$ & 100.0 & 100.0 & 100.0 & 3,600,000 & 1,820 & 5,000 \\
$\mathtt{ltl2dpa06}$ & 100.0 & 100.0 & 0.0 & 4,500,000 & 3,732 & 5,000 \\
$\mathtt{ltl2dpa07}$ & 100.0 & 100.0 & 7.9 & 4,500,000 & 3,693 & 15,000 \\
$\mathtt{ltl2dpa08}$ & 100.0 & 100.0 & 75.6 & 7,200,000 & 1,862 & 5,000 \\
$\mathtt{ltl2dpa09}$ & 1.7 & 100.0 & 8.6 & 8,100,000 & 3,709 & 5,000 \\
$\mathtt{ltl2dpa10}$ & 79.1 & 99.9 & 100.0 & 8,100,000 & 2,937 & 5,000 \\
$\mathtt{ltl2dpa12}$ & 24.6 & 100.0 & 100.0 & 8,100,000 & 2,648 & 5,000 \\
$\mathtt{ltl2dpa13}$ & 79.8 & 100.0 & 64.7 & 8,100,000 & 2,469 & 5,000 \\
$\mathtt{ltl2dpa14}$ & 100.0 & 100.0 & 100.0 & 2,700,000 & 917 & 5,000 \\
$\mathtt{ltl2dpa16}$ & 88.0 & 100.0 & 69.0 & 8,100,000 & 1,899 & 10,000 \\
$\mathtt{ltl2dpa17}$ & 63.2 & 100.0 & 0.0 & 8,100,000 & 3,889 & 5,000 \\
$\mathtt{ltl2dpa18}$ & 85.6 & 100.0 & 0.0 & 8,100,000 & 1,855 & 5,000 \\
$\mathtt{ltl2dpa19}$ & 49.2 & 100.0 & 68.7 & 8,100,000 & 4,417 & 10,000 \\
$\mathtt{ltl2dpa20}$ & 77.2 & 100.0 & 69.5 & 8,100,000 & 1,918 & 25,000 \\
$\mathtt{ltl2dpa23}$ & 100.0 & 100.0 & 100.0 & 1,800,000 & 1,787 & 5,000 \\
$\mathtt{ltl2dpa24}$ & 100.0 & 100.0 & 100.0 & 1,800,000 & 1,819 & 5,000 \\
$\mathtt{robot-tasks-real}$ & 100.0 & 0.0 & 100.0 & 900,000 & — & 5,000 \\
$\mathtt{robot\_grid}$ & 0.0 & 100.0 & 35.3 & 8,100,000 & 16,342 & 30,000 \\
$\mathtt{rw\_arbiter}$ & 100.0 & 100.0 & 37.5 & 5,400,000 & 1,860 & 5,000 \\
$\mathtt{solitary-real}$ & 99.2 & 100.0 & 100.0 & 8,100,000 & 16,781 & 5,000 \\
$\mathtt{tasks-real}$ & 99.2 & 0.0 & 100.0 & 8,100,000 & — & 5,000 \\
{$\mathtt{EscalatorBidirectional}$} & 0.0 & 0.0 & 100.0 & 8,100,000 & 37,258 & 5,000 \\
{$\mathtt{EscalatorCountingInit}$} & 33.6 & 100.0 & 100.0 & 8,100,000 & 3,766 & 5,000 \\
{$\mathtt{EscalatorNonReactive}$} & 100.0 & 0.0 & 100.0 & 900,000 & — & 5,000 \\
$\mathtt{MusicAppMotivating}$ & 2.1 & 100.0 & 2.6 & 8,100,000 & 8,564 & 5,000 \\
$\mathtt{OneCounterInRangeA3}$ & 1.8 & 100.0 & 100.0 & 8,100,000 & 5,490 & 5,000 \\
$\mathtt{TorcsAccelerating}$ & 100.0 & 100.0 & 100.0 & 900,000 & 1,799 & 5,000 \\
$\mathtt{TorcsSteeringSimple}$ & 0.2 & 100.0 & 100.0 & 8,100,000 & 4,726 & 5,000 \\
$\mathtt{TorcsSteeringSmart}$ & 0.0 & 100.0 & 22.2 & 8,100,000 & 6,314 & 25,000 \\
$\mathtt{arbiter-with-failure-}$ \\ $\mathtt{real}$ & 1.8 & 100.0 & 39.2 & 8,100,000 & 154,618 & 5,000 \\
$\mathtt{arbiter\_on\_inpchange}$ & 100.0 & 100.0 & 5.9 & 4,500,000 & 942 & 25,000 \\
$\mathtt{arbiter\_with\_cancel}$ & 100.0 & 100.0 & 0.1 & 1,800,000 & 1,780 & 10,000 \\
$\mathtt{batch-arbiter-real}$ & 20.4 & 100.0 & 100.0 & 8,100,000 & 333,487 & 5,000 \\
$\mathtt{chain-simple-10-}$ \\ $\mathtt{real}$ & 94.8 & 0.0 & 63.0 & 8,100,000 & — & 15,000 \\
$\mathtt{chain-simple-20-}$ \\ $\mathtt{real}$ & 93.0 & 0.0 & 63.7 & 8,100,000 & — & 15,000 \\
$\mathtt{chain-simple-30-}$ \\ $\mathtt{real}$ & 96.3 & 0.0 & 63.1 & 8,100,000 & — & 20,000 \\
$\mathtt{chain-simple-40-}$ \\ $\mathtt{real}$ & 94.6 & 0.0 & 64.2 & 8,100,000 & — & 20,000 \\
$\mathtt{chain-simple-5-real}$ & 96.0 & 0.0 & 42.7 & 8,100,000 & — & 5,000 \\
$\mathtt{chain-simple-50-}$ \\ $\mathtt{real}$ & 94.2 & 0.0 & 64.7 & 8,100,000 & — & 5,000 \\
$\mathtt{chain-simple-60-}$ \\ $\mathtt{real}$ & 94.4 & 0.0 & 64.0 & 8,100,000 & — & 5,000 \\
$\mathtt{chain-simple-70-}$ \\ $\mathtt{real}$ & 95.1 & 0.0 & 64.2 & 8,100,000 & — & 25,000 \\
$\mathtt{chain-simple-param-}$ \\ $\mathtt{70-real}$ & 97.7 & 0.0 & 100.0 & 8,100,000 & — & 5,000 \\
$\mathtt{generalized\_buffer}$ & 4.2 & 99.9 & 0.0 & 8,100,000 & 509,491 & 5,000 \\
$\mathtt{heim-double-x-real}$ & 99.4 & 0.0 & 100.0 & 8,100,000 & — & 5,000 \\
$\mathtt{infinite-race-real}$ & 97.0 & 100.0 & 100.0 & 8,100,000 & 9,989 & 5,000 \\
$\mathtt{infinite-race-unequal-}$ \\ $\mathtt{1-real}$ & 98.8 & 0.0 & 100.0 & 8,100,000 & — & 5,000 \\
$\mathtt{load\_balancer\_unreal2}$ & 72.8 & 100.0 & 100.0 & 8,100,000 & 3,961 & 5,000 \\
$\mathtt{ordered-visits-real}$ & 100.0 & 100.0 & 100.0 & 8,100,000 & 266,715 & 5,000 \\
$\mathtt{precise-reachability-}$ \\ $\mathtt{real}$ & 88.3 & 100.0 & 98.6 & 8,100,000 & 531,210 & 30,000 \\
$\mathtt{robot-cat-real-1d-}$ \\ $\mathtt{real}$ & 98.5 & 0.0 & 100.0 & 8,100,000 & — & 5,000 \\
$\mathtt{robot-grid-commute-}$ \\ $\mathtt{1d-real}$ & 83.3 & 0.0 & 100.0 & 8,100,000 & — & 5,000 \\
$\mathtt{robot-grid-reach-}$ \\ $\mathtt{1d-real}$ & 97.6 & 100.0 & 100.0 & 8,100,000 & 135,518 & 5,000 \\
$\mathtt{robot-grid-reach-}$ \\ $\mathtt{2d-real}$ & 78.9 & 0.0 & 93.3 & 8,100,000 & — & 15,000 \\
$\mathtt{robot-grid-reach-}$ \\ $\mathtt{repeated-with-obstacles-}$ \\ $\mathtt{1d-real}$ & 95.3 & 100.0 & 100.0 & 8,100,000 & 201,089 & 5,000 \\
$\mathtt{robot-grid-reach-}$ \\ $\mathtt{repeated-with-obstacles-}$ \\ $\mathtt{2d-real}$ & 93.3 & 0.0 & 100.0 & 8,100,000 & — & 5,000 \\
$\mathtt{robot\_collect\_samples\_v2-}$ \\ $\mathtt{real}$ & 99.0 & 0.0 & 100.0 & 8,100,000 & — & 5,000 \\
$\mathtt{storage-GF-64-real}$ & 92.3 & 96.3 & 82.2 & 8,100,000 & 1,053,092 & 20,000 \\
$\mathtt{thermostat-F-real}$ & 99.8 & 0.0 & 100.0 & 8,100,000 & — & 5,000 \\
$\mathtt{unordered-visits-}$ \\ $\mathtt{real}$ & 0.6 & 0.0 & 14.6 & 8,100,000 & — & 30,000 \\
$\mathtt{workstation\_resupply\_pb\_1\_pe\_}$ & 56.7 & 100.0 & 100.0 & 8,100,000 & 925 & 5,000 \\
$\mathtt{workstation\_resupply\_pb\_2\_pe\_}$ & 58.4 & 100.0 & 0.1 & 8,100,000 & 2,108 & 20,000 \\
$\mathtt{workstation\_resupply\_pb\_3\_pe\_}$ & 36.0 & 100.0 & 0.2 & 8,100,000 & 5,189 & 10,000 \\
\end{longtable}
\newpage
\subsection{Adaptive Arbitration Results}~\label{app:adaptive_arbitration_results}
\begin{figure}[h]
    \centering
    \includegraphics[width=1\linewidth]{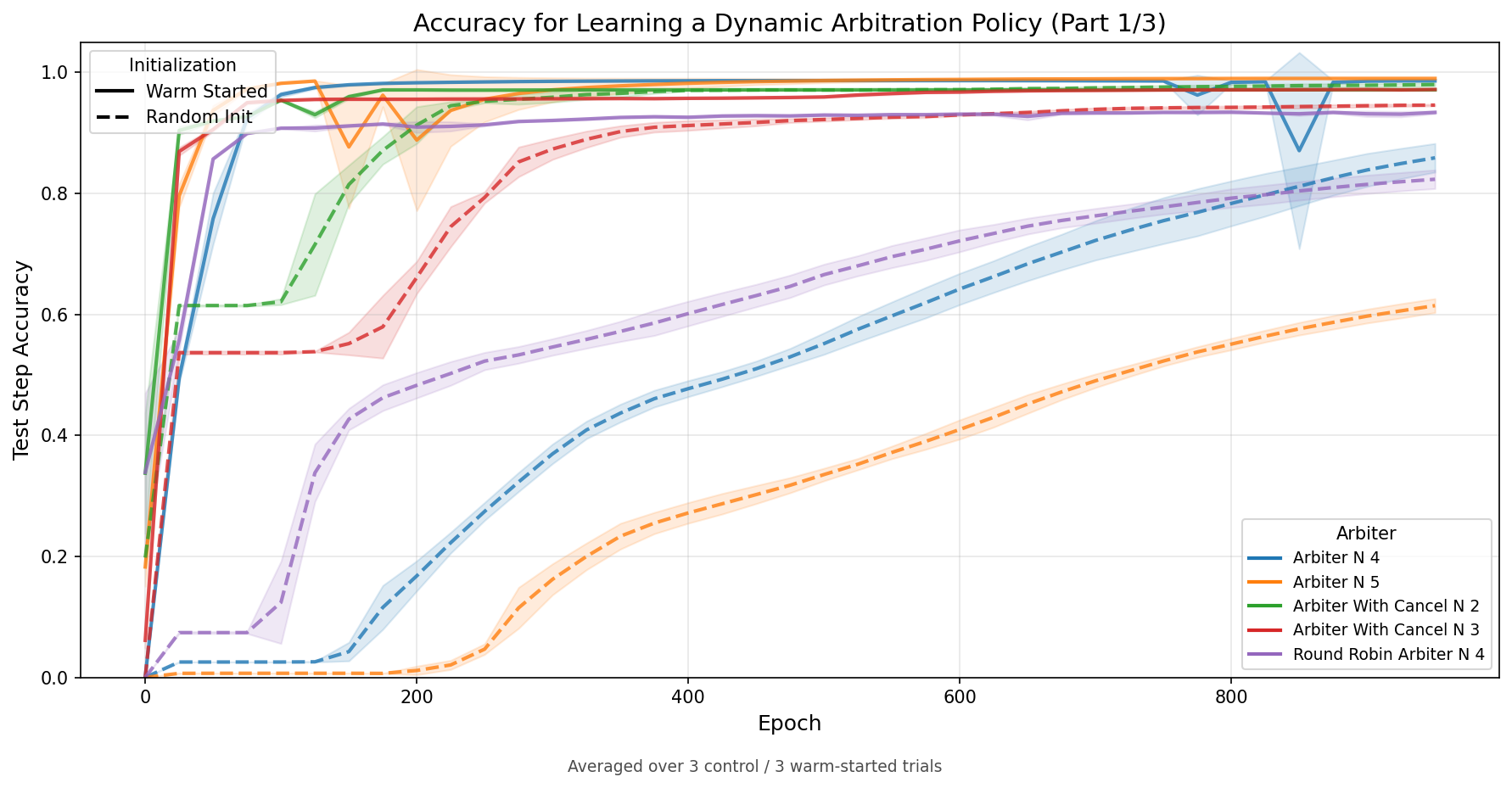}
    \caption{Adaptive arbiter results 1. Shown in main paper}
    \label{fig:adaptivearbiter_res_1}
\end{figure}
\begin{figure}[t]
    \centering
    \includegraphics[width=1\linewidth]{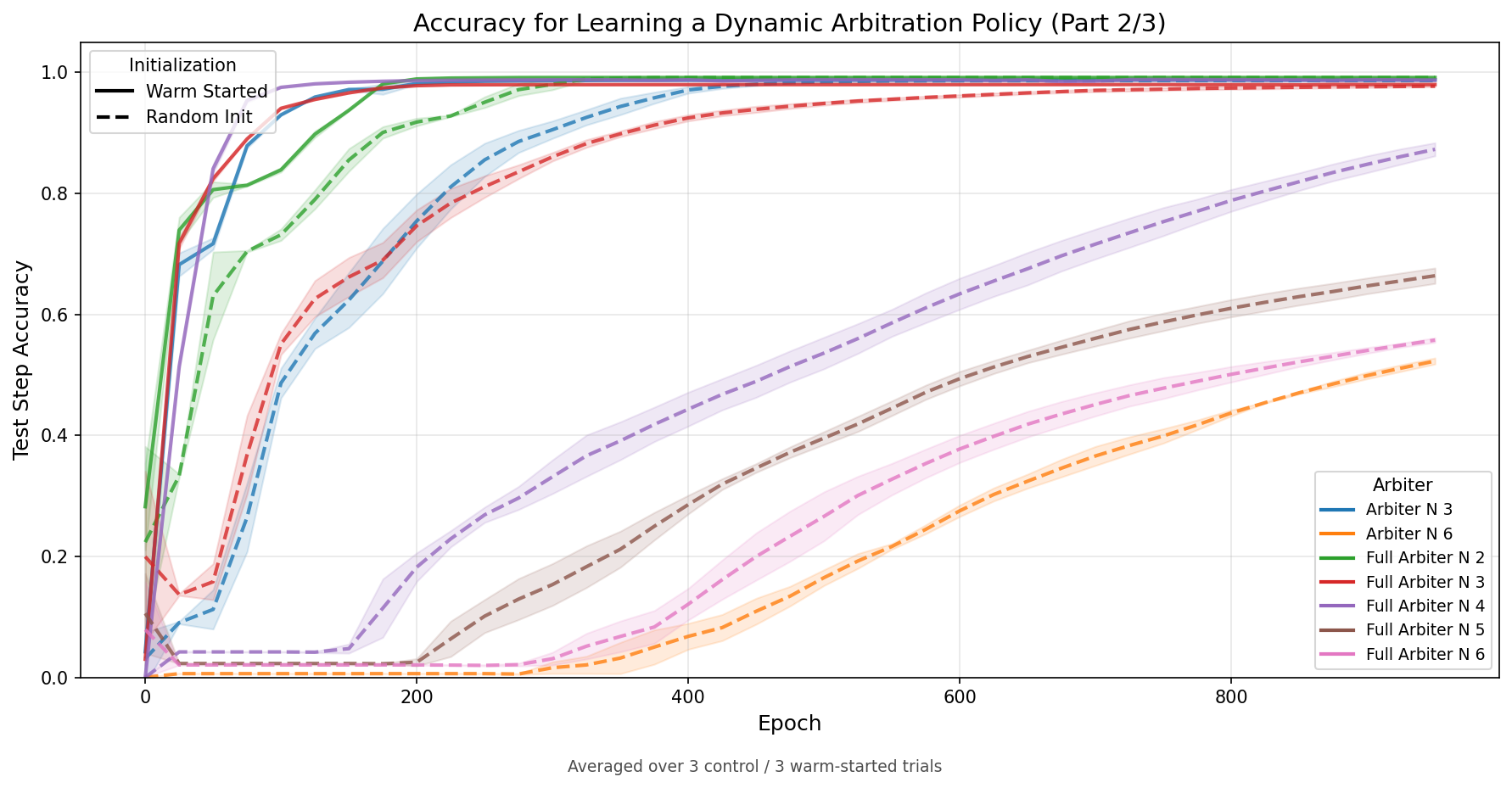}
    \caption{Adaptive arbiter results 2.}
    \label{fig:adaptivearbiter_res_2}
\end{figure}
\begin{figure}[t]
    \centering
    \includegraphics[width=1\linewidth]{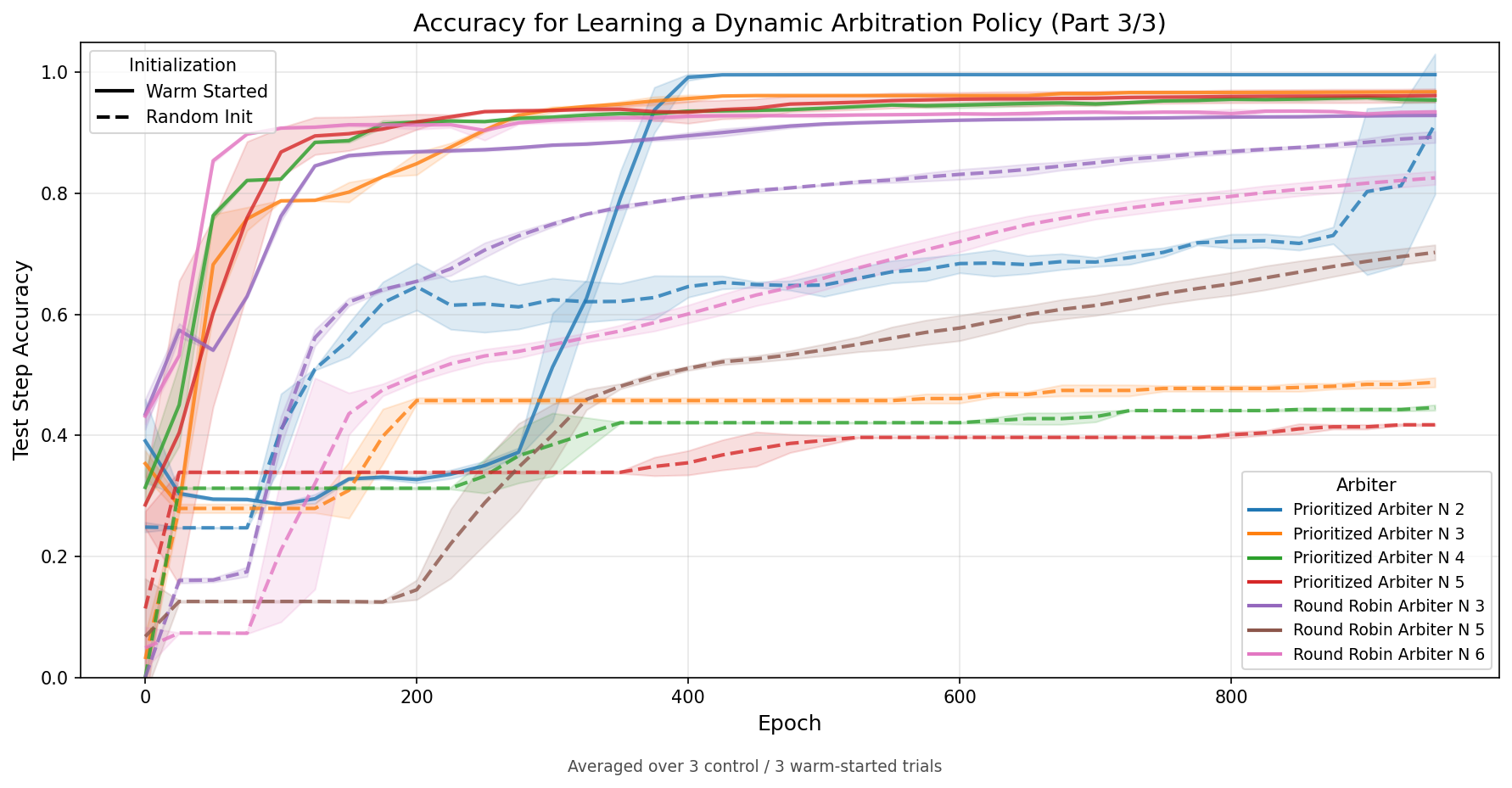}
    \caption{Adaptive arbiter results 3.}
    \label{fig:adaptive_arbiter_res_3}
\end{figure}
\end{document}